\documentclass[smallextended]{svjour3}

\journalname{Data Mining and Knowledge Discovery}

\usepackage{natbib}

\usepackage[ruled, vlined, nofillcomment, linesnumbered]{algorithm2e}

\usepackage{url}
\usepackage{amsmath}
\usepackage{amssymb}
\usepackage{subcaption}
\usepackage{booktabs}
\usepackage[english]{babel}
\usepackage{color}
\usepackage{bbding}

\usepackage[pdftitle={Discovering Episodes with Compact Minimal Windows},
            pdfauthor={Nikolaj Tatti},
			citecolor = black, linkcolor = black, colorlinks = true, bookmarksopen = true, bookmarksnumbered = true, hyperfootnotes = false]{hyperref}

\usepackage{tikz}
\usetikzlibrary{snakes,arrows,shapes,positioning,fit,calc}
\usepackage{pgfplots}

\usepackage{ifthen}
\newcommand{\set}[1]{\left\{#1\right\}}
\newcommand{\pr}[1]{\left(#1\right)}
\newcommand{\fpr}[1]{\mathopen{}\left(#1\right)}
\newcommand{\spr}[1]{\left[#1\right]}

\newcommand{\abs}[1]{{\left|#1\right|}}

\newcommand{\enset}[2]{\left\{#1 ,\ldots , #2\right\}}
\newcommand{\enpr}[2]{\pr{#1 ,\ldots , #2}}

\newcommand{\real}{\mathbb{R}}
\newcommand{\NP}{\textbf{NP}}
\newcommand{\funcdef}[3]{{#1}:{#2} \to {#3}}

\newcommand{\define}{\leftarrow}

\DeclareRobustCommand{\dispfunc}[2]{%
  \ensuremath{%
  \ifthenelse{\equal{\noexpand#2}{}}%
    {{#1}}%
    {{#1}\fpr{#2}}}}

\newcommand{\closure}[1]{\mathit{cl}\fpr{#1}}
\newcommand{\score}[1]{\dispfunc{\mathit{sc}}{#1}}
\newcommand{\sinks}[1]{\dispfunc{\mathit{sinks}}{#1}}

\newcommand{\lab}[1]{\dispfunc{\mathit{lab}}{#1}}
\newcommand{\sub}[1]{\dispfunc{\mathit{sub}}{#1}}
\newcommand{\stay}[1]{\dispfunc{\mathit{stay}}{#1}}
\newcommand{\parent}[1]{\dispfunc{\mathit{par}}{#1}}

\newcommand{\inc}[1]{\dispfunc{\mathit{in}}{#1}}
\newcommand{\simple}[1]{\dispfunc{\mathit{sm}}{#1}}
\newcommand{\co}[1]{\dispfunc{\mathit{co}}{#1}}

\newcommand{\mach}[1]{M_{\efam{#1}}}
\newcommand{\pre}[1]{\dispfunc{\mathit{pre}}{#1}}
\newcommand{\greedy}[1]{\dispfunc{\mathit{g}}{#1}}
\newcommand{\pgreedy}[1]{\dispfunc{\mathit{pg}}{#1}}
\newcommand{\mwmach}[1]{\dispfunc{\mathit{minm}}{#1}}
\newcommand{\moment}[1]{\dispfunc{\mathit{m}}{#1}}

\newcommand{\efam}[1]{\mathcal{#1}}

\newcommand{\prob}[1]{p\fpr{#1}}
\newcommand{\mean}[1]{\operatorname{E}\spr{#1}}
\newcommand{\means}[1]{\operatorname{E}\bigg[{#1}\bigg]}
\newcommand{\var}[1]{\operatorname{Var}\spr{#1}}
\newcommand{\cov}[1]{\operatorname{Cov}\spr{#1}}

\SetKwComment{tcpas}{\{}{\}}
\SetCommentSty{textnormal}
\SetArgSty{textnormal}
\SetKw{False}{false}
\SetKw{True}{true}
\SetKw{Null}{null}
\SetKwInOut{Output}{output}
\SetKwInOut{Input}{input}
\SetKw{AND}{and}
\SetKw{OR}{or}
\SetKw{Continue}{continue}

\renewenvironment{proof}{\begin{oldproof}}{{\hfill \ensuremath{\Box}}\end{oldproof}}

\def\clap#1{\hbox to 0pt{\hss#1\hss}}
\def\mathclap{\mathpalette\mathclapinternal}
\def\mathclapinternal#1#2{%
\clap{$\mathsurround=0pt#1{#2}$}%
}

\tikzstyle{ep} = [inner sep = 3pt, anchor = base west]
\tikzstyle{state} = [inner sep = 1pt, anchor = base west]
\tikzstyle{state2} = [inner sep = 1pt, anchor = base west, fill = yafcolor2, text = white]
\tikzstyle{label} = [text = yafcolor5]
\tikzstyle{sedge} = [->, >= latex]
\tikzstyle{sedge2} = [->, >= latex, densely dashed]
\tikzstyle{sedge3} = [->, >= latex, densely dashed, yafcolor3!40, ultra thick]
\tikzstyle{sedge4} = [->, >= latex, yafcolor1!60, thick]

\newcommand{\nodepath}[2]{%
\foreach \x [count=\xi from 1, count = \xprev from 0]  in {#1}{%
    \node[state, right= of #2\xprev.mid east, anchor = mid west, inner sep = 2pt] (#2\xi) {\x};%
}%
}

\newcommand{\edgepath}[1]{%
\foreach \x / \y in {#1}{%
	\draw[sedge, yafcolor5] (\x.mid east) ->  (\y.mid west);%
}%
}

\makeatletter
\pgfdeclareshape{event}{
\inheritsavedanchors[from=rectangle] 
\inheritanchorborder[from=rectangle]
\inheritanchor[from=rectangle]{center}
\inheritanchor[from=rectangle]{north}
\inheritanchor[from=rectangle]{south}
\inheritanchor[from=rectangle]{south east}
\inheritanchor[from=rectangle]{south west}
\inheritanchor[from=rectangle]{west}
\inheritanchor[from=rectangle]{east}
\backgroundpath{
\southwest \pgf@xa=\pgf@x \pgf@ya=\pgf@y
\northeast \pgf@xb=\pgf@x \pgf@yb=\pgf@y
\pgf@yc=\pgf@ya \advance\pgf@yc by 2pt
\pgfpathmoveto{\pgfpoint{\pgf@xa}{\pgf@yc}}
\pgfpathlineto{\pgfpoint{\pgf@xa}{\pgf@ya}}
\pgfpathlineto{\pgfpoint{\pgf@xb}{\pgf@ya}}
\pgfpathlineto{\pgfpoint{\pgf@xb}{\pgf@yc}}
}
}
\makeatother

\newcommand{\eventseq}[2]{%
\foreach \x [count=\xi from 1, count = \xprev from 0]  in {#1}{%
    \node[draw,shape=event, right=0 of #2\xprev.south east, anchor = south west, inner sep = 2pt, outer sep = 0pt] (#2\xi) {\x};%
}%
}

\pgfdeclarelayer{background}
\pgfdeclarelayer{foreground}
\pgfsetlayers{background,main,foreground}

\definecolor{yafaxiscolor}{rgb}{0.3, 0.3, 0.3}

\definecolor{yafcolor1}{rgb}{0.4, 0.165, 0.553}
\definecolor{yafcolor2}{rgb}{0.949, 0.482, 0.216}
\definecolor{yafcolor3}{rgb}{0.47, 0.549, 0.306}
\definecolor{yafcolor4}{rgb}{0.925, 0.165, 0.224}
\definecolor{yafcolor5}{rgb}{0.141, 0.345, 0.643}
\definecolor{yafcolor6}{rgb}{0.965, 0.933, 0.267}
\definecolor{yafcolor7}{rgb}{0.627, 0.118, 0.165}
\definecolor{yafcolor8}{rgb}{0.878, 0.475, 0.686}

\newlength{\yafaxispad}
\newlength{\yaftlpad}
\newlength{\yaflabelpad}
\newlength{\yafaxiswidth}
\newlength{\yafticklen}
\makeatletter
\def\pgfplots@drawtickgridlines@INSTALLCLIP@onorientedsurf#1{}
\makeatother

\newcommand{\yafdrawaxis}[4]{
	\pgfplotstransformcoordinatex{#1}\let\xmincoord=\pgfmathresult 
	\pgfplotstransformcoordinatex{#2}\let\xmaxcoord=\pgfmathresult 
	\pgfplotstransformcoordinatey{#3}\let\ymincoord=\pgfmathresult 
	\pgfplotstransformcoordinatey{#4}\let\ymaxcoord=\pgfmathresult 
	\pgfsetlinewidth{\yafaxiswidth} 
	\pgfsetcolor{yafaxiscolor}
	\pgfpathmoveto{\pgfpointadd{\pgfpointadd{\pgfplotspointrelaxisxy{0}{0}}{\pgfqpointxy{\xmincoord}{0}}}{\pgfqpoint{-0.5\yafaxiswidth}{\yafaxispad}}}
	\pgfpathlineto{\pgfpointadd{\pgfpointadd{\pgfplotspointrelaxisxy{0}{0}}{\pgfqpointxy{\xmaxcoord}{0}}}{\pgfqpoint{0.5\yafaxiswidth}{\yafaxispad}}}
	\pgfpathmoveto{\pgfpointadd{\pgfpointadd{\pgfplotspointrelaxisxy{0}{0}}{\pgfqpointxy{0}{\ymincoord}}}{\pgfqpoint{\yafaxispad}{-0.5\yafaxiswidth}}}
	\pgfpathlineto{\pgfpointadd{\pgfpointadd{\pgfplotspointrelaxisxy{0}{0}}{\pgfqpointxy{0}{\ymaxcoord}}}{\pgfqpoint{\yafaxispad}{0.5\yafaxiswidth}}}
	\pgfusepath{stroke}
}

\pgfplotscreateplotcyclelist{yaf}{%
{yafcolor1,mark options={scale=0.75},mark=o}, 
{yafcolor2,mark options={scale=0.75},mark=square},
{yafcolor3,mark options={scale=0.75},mark=triangle},
{yafcolor4,mark options={scale=0.75},mark=o},
{yafcolor5,mark options={scale=0.75},mark=o},
{yafcolor6,mark options={scale=0.75},mark=o},
{yafcolor7,mark options={scale=0.75},mark=o},
{yafcolor8,mark options={scale=0.75},mark=o}} 

\pgfplotsset{axis y line=left, axis x line=bottom,
	tick align=outside,
	tickwidth=\yafticklen,
	clip = false,
    x axis line style= {-, line width = 0pt, color=black!0},
    y axis line style= {-, line width = 0pt, color=black!0},
    x tick style= {line width = \yafaxiswidth, color=yafaxiscolor, yshift = \yafaxispad},
    y tick style= {line width = \yafaxiswidth, color=yafaxiscolor, xshift = \yafaxispad},
    x tick label style = {font=\scriptsize, yshift = \yaftlpad},
    y tick label style = {font=\scriptsize, xshift = \yaftlpad},
    every axis y label/.style = {at = {(ticklabel cs:0.5)}, rotate=90, anchor=center, font=\scriptsize, yshift = -\yaflabelpad},
    every axis x label/.style = {at = {(ticklabel cs:0.5)}, anchor=center, font=\scriptsize, yshift = \yaflabelpad},
    x tick label style = {font=\scriptsize, yshift = 1pt},
    grid = major,
    major grid style  = {dash pattern = on 1pt off 3 pt},
	every axis plot post/.append style= {line width=\yafaxiswidth} ,
	legend cell align = left,
	legend style = {inner sep = 1pt, inner ysep = 0pt, cells = {font=\scriptsize}},
	legend image code/.code={%
		\draw[mark repeat=2,mark phase=2,#1] 
		plot coordinates { (0cm,0cm) (0.15cm,0cm) (0.2cm,0cm) };%
	} 
}

\begin{document}
\title{Discovering Episodes with Compact Minimal Windows}

\author{Nikolaj Tatti}
\institute{
Nikolaj Tatti \at
ADReM, University of Antwerp, Belgium\\
DTAI, KU Leuven, Belgium\\
HIIT, Aalto University, Finland\\
\email{nikolaj.tatti@aalto.fi}
}

\maketitle

\begin{abstract}
Discovering the most interesting patterns is the key problem in the field of
pattern mining.  While ranking or selecting patterns is well-studied for
itemsets it is surprisingly under-researched for other, more complex, pattern
types.

In this paper we propose a new quality measure for episodes. An episode is
essentially a set of events with possible restrictions on the order of events.
We say that an episode is significant if its occurrence is abnormally compact,
that is, only few gap events occur between the actual episode events, when
compared to the expected length according to the independence model. We can
apply this measure as a post-pruning step by first discovering frequent
episodes and then rank them according to this measure.

In order to compute the score we will need to compute the mean and the
variance according to the independence model.  As a main technical contribution
we introduce a technique that allows us to compute these values.  Such a task
is surprisingly complex and in order to solve it we develop intricate finite
state machines that allow us to compute the needed statistics. We also show
that asymptotically our score can be interpreted as a $P$-value.  In our
experiments we demonstrate that despite its intricacy our ranking is fast: we can rank tens of
thousands episodes in seconds. Our experiments with text data demonstrate
that our measure ranks interpretable episodes high.
\end{abstract}


\section{Introduction}
Discovering the most interesting patterns is the key problem in the field of pattern mining.
While ranking or selecting patterns is well-studied for itemsets, a
canonical and arguably the easiest pattern type, it is surprisingly
under-researched for other, more complex, pattern types. 

Discovering episodes, frequent patterns from an event sequence has been a
fruitful and active field in pattern mining since their original introduction
by~\citet{mannila:97:discovery}. Essentially, an episode is a set of events that
should occur close to each other (gaps are allowed) possibly with some
constraints on the order of the occurrences, see Section~\ref{sec:prel} for
full definition.  While the concept of support for itemsets is straightforward,
it is simply the number of transactions containing the pattern, defining a
support for episodes is more complex. The most common way of defining a support
is to slide a window of fixed size over the sequence and count in how many
windows the pattern occurs.  Such a measure is monotonically decreasing and
hence all frequent episodes can be found using \textsc{APriori} approach given
by~\citet{mannila:97:discovery}.  Alternatively we can consider counting minimal
windows, that is finding and counting the most compact windows that contain the
episode.

The common wisdom is that finding frequent patterns is not enough.  Discovering
frequent patterns with high threshold will result to trivial patterns, omitting
many interesting patterns, while using a low threshold will result in a pattern
explosion.  This phenomenon has led to many ranking methods for itemsets, the
most well-studied pattern type. Unlike for itemsets, ranking episodes is
heavily under-developed. Existing statistical approaches for ranking episodes
are mostly based on the number of fixed-size windows (see more detailed discussion in
Section~\ref{sec:related}). However, a natural way of measuring the goodness
of an episode is the average length of its instances---a good episode should have
compact minimal windows. Hence, our goal and contribution is a measure
based directly on the average length of minimal windows.

The most straightforward and common way to measure significance for itemsets
is to compare the observed support, the number of transactions in which all
attributes co-occur, against the independence model: if the observed support
deviates a lot from the expectation, we consider the itemset important.
In this paper we use the same principle and
propose an interestingness measure for an episode by comparing
the observed lengths of minimal windows of the episode against the expectation
computed from the independence model. Given a set of episodes we can now apply
our measure to each episode and rank the episodes, placing episodes with the most
abnormal minimal windows on top.
While this is an easy task for itemsets, computing statistics turns out to be
complex for episodes.

We define our score as follows: given an episode $G$, we assign a weight to
each minimal window of $G$ based on how long it is. The weight
will be large for small windows and small for large windows.  To compute the
expected weight we assume that for each symbol we have a probability of its
occurrence in the sequence. We then compute the expected weight based on a
model in which the symbols are independent of each other.  We say that the
episode is significant if the observed average weight is abnormally large, that
is, the minimal windows are abnormally short.

\begin{example}
\label{ex:toy1}
Assume that we have an alphabet of size $3$, $\Sigma = \set{a, b, c}$. Assume
that the probabilities for having a symbol are $\prob{a} = 1/2$, $\prob{b} =
1/4$, and $\prob{c} = 1/4$. Let $G$ be a serial episode $a \to b$.  Then $s$ is
a minimal window for $G$ if and only if it has a form $ac\cdots cb$.  Hence the
probability of a random sequence $s$ of length $k$ to be a minimal window for
$G$ is equal to 
\[
\begin{split}
p(s \text{ is a minimal window  of } G, \abs{s} = k) = \frac{1}{2} \times \frac{1}{4} \times \frac{1}{4^{k - 2}}.
\end{split}
\]
We are interested in a probability of a minimal window having length $k$.
To get this we divide the joint probability by the probability 
\[
	p(s \text{ is a minimal window  of } G) = \sum_{k = 2}^\infty  \frac{1}{2} \times \frac{1}{4} \times \frac{1}{4^{k - 2}}  =  1/6.
\]
Using this normalisation we get that the probability of a minimal window having 
length $k$ is equal to
\[
p(\abs{s}  = k \mid s \text{ is a minimal window of } G) = 3/4 \times 1/4^{k - 2},
\]
for $k \geq 2$, and $0$ otherwise. If we now weight minimal windows with an
exponential decay, say, $1/2^{\abs{s}}$, then the expected weight is equal to
$3 / 14 \approx 0.2$.  On the other hand, assume that we have a sequence $s =
abcacbcababcab$.  There are $4$ minimal windows of length $2$ and one minimal
window of length $3$.  Hence, the observed average weight is $(4 \times 1/2^2 +
1/2^3) /5 = 0.225$ suggesting that the minimal windows are more compact than
what the independence model implies.
\end{example}

Computing the needed statistics turns out to be a surprisingly complex problem.
We attack this problem in Section~\ref{sec:machine} by introducing a certain
finite state machine having episodes as the nodes. Then using this structure we
are able to compute the statistics recursively, starting from simple
episodes and moving towards more complex ones.

Our recipe for the mining process is as follows: Given the sequence we first
split the sequence in two. The first sequence is used for discovering candidate
episodes, in our case episodes that have a large number of minimal windows.  Luckily, this condition is
monotonically decreasing and we can mine these episodes using a standard
\textsc{APriori} method.  We also compute the needed probabilities of individual
events from the first sequence. Once we have discovered candidate episodes and
have computed the expectation, we compare the expected weight against the average
observed weight from the \emph{second} sequence using a simple $Z$-score. This
step allows us to prune uninteresting episodes, which is in our case episodes that obey
the independence model.


The rest of the paper is structured as follows. In
Section~\ref{sec:prel} we introduce the preliminary
definitions and notation.  We introduce our
method for evaluating the difference between the observed windows and the
independence model in Section~\ref{sec:window}.
In Sections~\ref{sec:machine}--\ref{sec:moments} we lay out our
approach for computing the independence model.
We present the related work in Section~\ref{sec:related}.
Our experiments are given in Section~\ref{sec:experiments}
and we conclude our work with discussion in
Section~\ref{sec:conclusions}.
All proofs are given in Appendix.

\section{Preliminaries and Notation}
\label{sec:prel}

We begin by presenting preliminary concepts and notations that will be used
throughout the rest of the paper.

A \emph{sequence} $s = \enpr{s_1}{s_L}$ is a string of symbols coming from a
finite \emph{alphabet} $\Sigma$, that is, we have $s_i \in \Sigma$.
Given a sequence $s$ and two indices $i$ and
$j$, such that $i \leq j$, we denote by $s[i, j] = \enpr{s_i}{s_j}$ 
a sub-sequence of $s$.

An episode $G$ is represented by an acyclic directed graph with labelled nodes,
that is $G = (V, E, \lab{})$, where $V = \enpr{v_1}{v_K}$ is the set of nodes, $E$
is the set of directed edges, and $\lab{}$ is the function $\funcdef{\lab{}}{V}{\Sigma}$,
mapping each node $v_i$ to its label.

Given a sequence $s$ and an episode $G$ we say that $s$ \emph{covers}  the
episode if there is an \emph{injective} map $f$ mapping each node $v_i$ to a valid
index such that the node and the corresponding sequence element have the same
label, $s_{f(v_i)} = \lab{v_i}$, and that if there is an edge $(v_i, v_j) \in E$, then
we must have $f(v_i) < f(v_j)$. In other words, the parents of the node $v_i$
must occur in $s$ before $v_i$. 
Traditional episode mining is based on searching episodes that are covered by
sufficiently many sub-windows of certain fixed size.

\begin{example}
Consider an episode given in Figure~\ref{fig:toycover}.
This episode has 4 nodes labelled as $a$, $b$, $c$, and $d$, and requires 
that $a$ must come first, followed by $b$ and $c$ in arbitrary order, and finally followed by 
$d$. Figure~\ref{fig:toycover} also shows an example of a sequence that covers the episode.

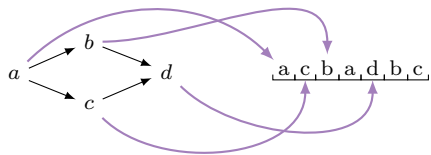
\begin{figure}[htb!]
\begin{center}
\begin{tikzpicture}
\node[ep] (e1) at (0, 0) {$a$};
\node[ep] (e2) at (1, 0.4) {$b$};
\node[ep] (e3) at (1, -0.4) {$c$};
\node[ep] (e4) at (2, 0) {$d$};
\draw[sedge] (e1) -> (e2);
\draw[sedge] (e1) -> (e3);
\draw[sedge] (e2) -> (e4);
\draw[sedge] (e3) -> (e4);

\node[right=1cm of e4] (n0) {};
\eventseq{a, c, b, a, d, b, c}{n}
\draw[sedge4] (e1) edge [out=45,in=135] (n1);
\draw[sedge4] (e2) edge [out=0,in=90] (n3);
\draw[sedge4] (e3) edge [out=315,in=270] (n2);
\draw[sedge4] (e4) edge [out=315,in=270] (n5);
\end{tikzpicture}
\end{center}

\caption{A toy episode with 4 nodes and an example of a sequence covering the episode}
\label{fig:toycover}
\end{figure}

\end{example}

An elementary theorem says that in a directed acyclic graph there exists a sink,
a node with no outgoing edges. We denote the set of sinks by $\sinks{G}$.
Given an episode $G$ and a node $v$, we define $G - v$ to be the sub-episode
obtained from $G$ by removing $v$, and the incident edges. 

Given an episode $G$ we define a set of \emph{prefix} episodes by
\[
	\pre{G} = \set{G} \cup \bigcup_{v \in \sinks{G}} \pre{G - v},
\]
that is, a prefix episode $H$ is a subepisode of $G$ such that if $v_i$ is contained in $H$,
then all parents (in $G$) of $v_i$ are also contained in $H$.

\begin{example}
Episode given in Figure~\ref{fig:toycover} has 6 prefix episodes. Among of these 6 episodes
one is empty, the remaining 5 episodes are given in Figure~\ref{fig:toyprefix}.

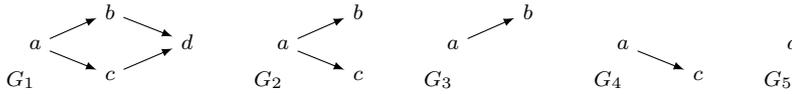
\begin{figure}[htb!]
\begin{center}

\begin{tikzpicture}
\node at (0, -0.4) {$G_1$};
\node[ep] (e1) at (0, 0) {$a$};
\node[ep] (e2) at (1, 0.4) {$b$};
\node[ep] (e3) at (1, -0.4) {$c$};
\node[ep] (e4) at (2, 0) {$d$};
\draw[sedge] (e1) -> (e2);
\draw[sedge] (e1) -> (e3);
\draw[sedge] (e2) -> (e4);
\draw[sedge] (e3) -> (e4);
\end{tikzpicture}\hspace{5mm}
\begin{tikzpicture}
\node at (0, -0.4) {$G_2$};
\node[ep] (e1) at (0, 0) {$a$};
\node[ep] (e2) at (1, 0.4) {$b$};
\node[ep] (e3) at (1, -0.4) {$c$};
\draw[sedge] (e1) -> (e2);
\draw[sedge] (e1) -> (e3);
\end{tikzpicture}\hspace{5mm}
\begin{tikzpicture}
\node at (0, -0.4) {$G_3$};
\node[ep] (e1) at (0, 0) {$a$};
\node[ep] (e2) at (1, 0.4) {$b$};
\draw[sedge] (e1) -> (e2);
\end{tikzpicture}\hspace{5mm}
\begin{tikzpicture}
\node at (0, -0.4) {$G_4$};
\node[ep] (e1) at (0, 0) {$a$};
\node[ep] (e3) at (1, -0.4) {$c$};
\draw[sedge] (e1) -> (e3);
\end{tikzpicture}\hspace{5mm}
\begin{tikzpicture}
\node at (0, -0.4) {$G_5$};
\node[ep] (e1) at (0, 0) {$a$};
\end{tikzpicture}

\end{center}
\caption{Non-empty prefix episodes of an episode given in Figure~\ref{fig:toycover}}
\label{fig:toyprefix}
\end{figure}
\end{example}

\section{Minimal Windows of Episodes}
\label{sec:window}
Traditionally, discovering episodes from a single long sequence can be done in
two ways. The first approach is to slide a window of fixed sized over the
window and count the number of windows in which the episode occurs. The second
approach is to count the number of minimal windows. The goal of this paper is
to build a measure based minimal windows. If the
statistic is abnormal, then we consider this pattern important.

In order to make the preceding discussion more formal, let $G$ be an episode,
and let $s$ be a sequence. We say that $s$ is a \emph{minimal window} for $G$
if $G$ is covered by $s$ but not by any proper sub-window of $s$.
In this paper we are interested in discovering episodes that have abnormally
compact minimal windows, a natural way of defining 
the significance of an episode.

\begin{example}
Consider a toy episode given in Figure~\ref{fig:toycover}. The sequence given
in Figure~\ref{fig:toycover} covers the episode but it not a minimal window.
However, if we remove 2 last symbols from the sequence, then the sequence becomes
a minimal window.
\end{example}

\begin{example}
Consider a serial episode $a \to b$, that is a pattern stating event $a$ should
be followed by an event $b$, and two sequences '$abababababababab$' and
'$abacbadbaxbaybab$'. If we fix the length of a window to be $6$ (or larger),
then the number of windows covering the episode will be the same for the both
sequences. In fact, in this case all windows will contain the episode.
However, occurrences of the episode in these sequences are different.  In the
first sequence, all minimal windows are of length $2$, while in the second
sequence, we have 2 minimal windows of length $2$ and $4$ minimal windows of
length $3$. Our intuition is that $a \to b$ should be considered more
significant in the first sequence than in the second.
\end{example}

Our goal in this paper is to design a measure that will indicate if
the minimal windows are significantly compact. One approach would be to measure
the average length of minimal windows. However, this ratio is susceptible to
the variance in large minimal windows: consider that we have two minimal
windows: the first is of length $10$ and the other is of length $1000$.  Then
the length of the second window dominates the average length, even though the
first window is more interesting.  In order to counter this phenomenon we
suggest using the following statistic. Assume that we are given a parameter $0
< \rho < 1$.  Let $s$ be a minimal window for $G$. We define the \emph{weight} of
a window to be $\rho^{\abs{s}}$. Compact windows will have a large value
whereas large windows will have a small value. Let $r$ be the average
weight of all minimal windows for $G$.

We are interested in testing whether $r$ is significantly large. In order to do
that, let $s$ be a random sequence and define a random variable $Y_i = a$  if
$s[i, a]$ is a minimal window, if there is no such $a$ we define $Y_i = 0$.
Define also $X_i = Y_i > 0$ to be the indicator whether $s$ has a minimal
window of $G$ starting at $i$th index.

We suggest using the following statistic.
Given a parameter $0 < \rho < 1$, we define $Z_i = X_i\rho^{Y_i - i + 1}$.
Then $r$ is an estimate of a statistic $\sum_{i = 1}^\infty Z_i / \sum_{i = 1}^\infty X_i$. 

We will show that there is $\mu$ and $\sigma$ such that
\[
	\sqrt{L}\big(\sum_{i = 1}^L Z_i / \sum_{i = 1}^L X_i - \mu\big)
\]
approaches a normal distribution $N(0, \sigma^2)$. This suggest to define a
measure $\score{G} = \sqrt{L}(r - \mu) / \sigma$. This is simply a
$Z$-normalisation of the statistic $r$.

We can also compute $\Phi\pr{-\score{G}}$, where $\Phi$ is the cumulative density
function of the standard normal distribution $N(0, 1)$, and interpret this
quantity as a $P$-value. However, this interpretation is problematic mainly
because the normal distribution estimate is only accurate asymptotically.

Hence, we only consider $\score{G}$ merely as a ranking measure. Nevertheless,
this measure makes a lot of sense: it measures how much the observed value
deviates from the expectation, a common approach in ranking patterns, and
it also takes the account the uncertainty of the measure.

In order to achieve our goal, we need to perform two steps
\begin{enumerate}
\item We need to show that $\score{G}$ converges into $N(0, 1)$
\item We need to compute $\mu$ and $\sigma^2$ that are needed for $\score{G}$.
\end{enumerate}

Both of these steps are non-trivial. Proving asymptotic normality is difficult
because $X_i$, $Z_i$, and $Y_i$ are not independent, hence we will have to show that the sequence
is mixing fast enough. Computing $\mu$ and $\sigma^2$ will require a set of
recursive equations. The remaining theoretical sections are devoted to proving
asymptotic normality and computing the mean and the variance.

\section{Detecting Minimal Windows}
\label{sec:machine}
In this and the next section we establish our main theoretical contribution, which is
how to compute $\score{G}$.

We divide our task as follows: In Section~\ref{sec:construct} we build a finite
state machine recognising when an episode is covered. In
Section~\ref{sec:simple} we modify this machine so that we can use it for
subsequent statistical calculations. Using this machine as a base we construct
in Section~\ref{sec:minmach} a machine that is able to recognise a minimal
window of $G$.

\subsection{Constructing finite state machine}
\label{sec:construct}

We begin by constructing a finite state machine that recognises the coverage
of an episode.

In this paper, a \emph{finite state machine} (or simply a machine) $M$ is a DAG
with labelled edges and a single source. We allow multiple edges between two
nodes.

Given a state $x$ in $M$ we say that $s$ \emph{covers} $x$ if there is a
subsequence $t = \enpr{s_{i_1}}{s_{i_N}}$ such that $x$ can be reached from the
source node using $t$ as an input. 

Given an episode $G$, we define a machine $M_G$ to be a DAG containing prefix
graphs as nodes $V(M_G) = \set{x_H \mid H \in \pre{G}}$. We add an edge $e =
(x_H, x_F)$ if and only if there is a sink node $v \in V(G)$ such that $H = F -
v$.  We label edge $e$ with the label of $v$, $\lab{e} = \lab{v}$.

\begin{example}
Consider an episode $G$ given in Figure~\ref{fig:toy1:a}. The corresponding
machine $M_G$ is given in Figure~\ref{fig:toy1:b}. Sink state $x_6$
corresponds to episode $G$ and source state $x_1$ corresponds to the
empty episode. Intermediate state $x_5$ corresponds to $G_2$ given in Figure~\ref{fig:toyprefix},
$x_3$ corresponds to $G_3$, $x_4$ corresponds to $G_4$, and $x_2$ corresponds to $G_5$.

\begin{figure}
\begin{center}
\subcaptionbox{Episode $G$\label{fig:toy1:a}}{
\begin{tikzpicture}
\node[ep] (n1) at (0, 0) {$a$};
\node[ep] (n2) at (1, 0.4) {$b$};
\node[ep] (n3) at (1, -0.4) {$c$};
\node[ep] (n4) at (2, 0) {$d$};
\draw[sedge] (n1) -> (n2);
\draw[sedge] (n1) -> (n3);
\draw[sedge] (n2) -> (n4);
\draw[sedge] (n3) -> (n4);
\end{tikzpicture}
}\hspace{1cm}
\subcaptionbox{Machine $M_G$\label{fig:toy1:b}}{
\begin{tikzpicture}
\node[state] (n1) at (0, 0) {$x_1$};
\node[state] (n2) at (1, 0) {$x_2$};
\node[state] (n3) at (2, 0.4) {$x_3$};
\node[state] (n4) at (2, -0.4) {$x_4$};
\node[state] (n5) at (3, 0) {$x_5$};
\node[state] (n6) at (4, 0) {$x_6$};
\draw[sedge] (n1) -> node[label, above] {$a$} (n2);
\draw[sedge] (n2) -> node[label, above] {$b$} (n3);
\draw[sedge] (n2) -> node[label, above] {$c$} (n4);
\draw[sedge] (n3) -> node[label, above] {$c$} (n5);
\draw[sedge] (n4) -> node[label, above] {$b$} (n5);
\draw[sedge] (n5) -> node[label, above] {$d$} (n6);
\end{tikzpicture}}
\end{center}
\caption{Toy example of an episode $G$ and the corresponding machine $M_G$.}
\end{figure}
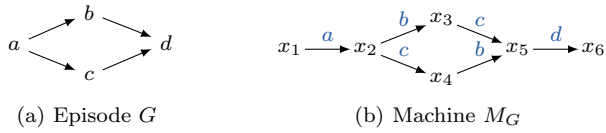

\end{example}

Comparing the definition of coverage of a state in $x$ and the definition
of a coverage for episodes gives immediately the following proposition.

\begin{proposition}
\label{prop:cover}
Given an episode $G$, a sequence $s$ covers an episode $H \in \pre{G}$ if and
only if $s$ covers the corresponding state $x_H$ in $M_G$. 
\end{proposition}

\subsection{Making Simple Machines}
\label{sec:simple}

In order to be able to compute the needed probabilities in subsequent sections,
a machine need to have a crucial property. We say that machine $M$ is
\emph{simple} if each state in $M$ does not multiple incoming edges with the
same label. If we reverse the direction of edges, then simplicity is equivalent
to a finite state machine being deterministic.

In general, $M_G$ is not simple. If an episode $G$ contains two nodes, say
$v_i$ and $v_j$ with the same label such that $v_i$ is not an ancestor of $v_j$
and vice versa, then there is a state $x_H$ in $M_G$, where $H$ is a prefix episode
having $v_i$ and $v_j$ as sinks will have (at least) two incoming edges with the
same label (see Figures~\ref{fig:toy2:a}--\ref{fig:toy2:b}).

Luckily, we can transform $M_G$ into a simple machine. This transformation is
almost equivalent to a process of making a non-deterministic finite state machine
to deterministic.

In order to make this formal, let us first give some definitions. Assume that we
are given a machine $M$.
Given a state $x$ in $M$, we define
\[
	\inc{x} = \set{\lab{e} \mid e = (y, x) \in E(M)}
\]
to be the set of labels of all incoming edges.  If $X$ is a subset of states in $M$,
then we write $\inc{X} = \bigcup_{x \in X} \inc{x}$.

Let $X$
be a subset of states in $M$ and let $a$ be a label.  We define
\[
	\sub{X; a} = \set{y  \mid e = (y, x) \in E(M), \lab{e} = a, x \in X }
\]
to be the union set of parents of each $v \in X$ connected with an edge having
the label $a$. We also define
\[
	\stay{X; a} = \set{x \in X  \mid a \notin \inc{x}}
\]
to be the set of states that have no incoming edge with a label $a$.

Let $i$ be the (unique) source state in $M$. We define
\[
	\parent{X; a} =
	\begin{cases}
		\sub{X; a} \cup \stay{X; a} & \text{if } i \notin \sub{X; a} \\
		\set{i}    & \text{if } i \in \sub{X; a}. \\
	\end{cases}
\]
Finally, we define a closure of $X$
inductively to be the collection of sets of states
\[
	\closure{X} = \set{X} \cup \bigcup_{a \in \inc{X}} \closure{\parent{X; a}}.
\]

We are now ready to define a simple machine $\simple{M}$. The states of this machine are
\[
	V(\simple{M}) = \bigcup_{x \in \sinks{M}} \closure{\set{x}}.
\]
An edge $e = (X, Y)$ with a label $a$ is in $E(\simple{M})$ if and only if $a \in \inc{Y}$ and $X = \parent{Y; a}$. 
Since, for each $a$, there is only one $X$ such that $X = \parent{Y; a}$, it follows that $\simple{M}$ is simple.

\begin{example}
A machine $M_G$ given in Figure~\ref{fig:toy2:b} is not simple since the state
$x_4$ has two incoming edges with $a$, each edge correspond to either one of
$a$. In order to obtain $\simple{M_G}$, we first observe that the nodes are
\[
\begin{split}
	\set{\set{x_6}} \cup \closure{\set{x_4}} \cup \closure{\set{x_5}} & = \set{\set{x_6}, \set{x_4}, \set{x_5}} \cup \closure{\set{x_2, x_3}} \cup \closure{\set{x_5}} \\
		& = \set{\set{x_6}, \set{x_4}, \set{x_5}, \set{x_2, x_3}, \set{x_3}, \set{x_1}}.
\end{split}
\]
This final machine is given in Figure~\ref{fig:toy2:c}.
Note that $\simple{M_G}$ is simple since parents of $x_4$ are grouped together.

\begin{figure}
\begin{center}
\subcaptionbox{$G$\label{fig:toy2:a}}{
\begin{tikzpicture}
\node[ep] (n1) at (0, 0) {$a$};
\node[ep] (n2) at (0, 0.4){$a$};
\node[ep] (n3) at (1, 0){$b$};
\draw[sedge] (n1.mid east) -> (n3.mid west);
\end{tikzpicture}}\hspace{0.5cm}
\subcaptionbox{Machine $M_G$\label{fig:toy2:b}}{
\begin{tikzpicture}
\node[state] (n1) at (0, 0) {$x_1$};
\node[state] (n2) at (1, 0.4){$x_2$};
\node[state] (n3) at (1, -0.4) {$x_3$};
\node[state] (n4) at (2, 0) {$x_4$};
\node[state] (n5) at (2, -0.8) {$x_5$};
\node[state] (n6) at (3, -0.4) {$x_6$};
\draw[sedge] (n1) -> node[label, above] {$a$} (n2);
\draw[sedge] (n1) -> node[label, above] {$a$} (n3);
\draw[sedge] (n2) -> node[label, above] {$a$} (n4);
\draw[sedge] (n3) -> node[label, above] {$a$} (n4);
\draw[sedge] (n3) -> node[label, above] {$b$} (n5);
\draw[sedge] (n4) -> node[label, above] {$b$} (n6);
\draw[sedge] (n5) -> node[label, above] {$a$} (n6);
\end{tikzpicture}}\hspace{0.5cm}
\subcaptionbox{Machine $\simple{M_G}$\label{fig:toy2:c}}{
\begin{tikzpicture}
\node[state] (n1) at (0, 0) {$x_1$};
\node[state] (n2) at (1, 0.4){$x_2x_3$};
\node[state] (n3) at (1, -0.4) {$x_3$};
\node[state] (n4) at (2.5, 0.4) {$x_4$};
\node[state] (n5) at (2.5, -0.4) {$x_5$};
\node[state] (n6) at (3.5, 0) {$x_6$};
\draw[sedge] (n1) -> node[label, above] {$a$} (n2);
\draw[sedge] (n1) -> node[label, above] {$a$} (n3);
\draw[sedge] (n2) -> node[label, above] {$a$} (n4);
\draw[sedge] (n3) -> node[label, above] {$b$} (n5);
\draw[sedge] (n4) -> node[label, above] {$b$} (n6);
\draw[sedge] (n5) -> node[label, above] {$a$} (n6);
\end{tikzpicture}}
\end{center}
\caption{Toy example of an episode $G$ and the corresponding machines $M_G$ and $\simple{M_G}$.}
\end{figure}
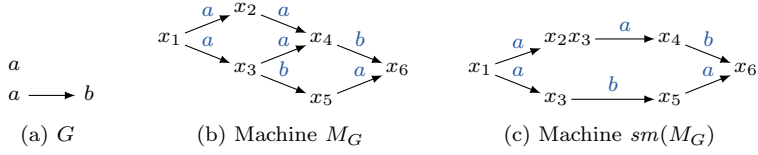

\end{example}

The following proposition reveals the expected result between $M$ and
$\simple{M}$.

\begin{proposition}
\label{prop:simplecover}
Let $M$ be a machine.  Let $X = \enset{x_1}{x_N}$ be a state in
$\simple{M}$. Then a sequence $s$ covers $V$ if and only if $s$ covers at least
one $x_i$.
\end{proposition}

The coverage of a machine is based on subsequences and working with
subsequences is particularly difficult since there may be several subsequences
that cover episode $G$, which leads to difficulties when computing
probabilities. 

Instead of working with subsequences directly, we will define a greedy function.
Assume that we are given a simple machine $M$. Let $x \in M$ be a state and
let $s = \enpr{s_1}{s_L}$ be a sequence. We define a greedy function recursively
\[
	\greedy{x, s} =
	\begin{cases}
		x & \text{if } L = 0, \\
		\greedy{y, s[1, L - 1]}  & \text{if there is } (y, x) \text{ such that } \lab{(y, x)} = s_L, \\
		\greedy{x, s[1, L - 1]}  & \text{otherwise}.
	\end{cases}
\]
In other words, the greedy function descends to parent states as fast as possible.

\begin{example}
Consider a machine $M_G$ given in Figure~\ref{fig:toy1:b} and
sequence $s = acbadbc$ given in Figure~\ref{fig:toycover}.
We have
\begin{align*}
	\greedy{x_6, acbadbc} & = \greedy{x_6, acbadb} = \greedy{x_6, acbad} = \greedy{x_5, acba}\\
	& = \greedy{x_5, acb} = \greedy{x_4, ac} = \greedy{x_2, a} = \greedy{x_1, \emptyset}  = x_1.
\end{align*}
\end{example}

The example suggests that a sequence covers an episode if the greedy function
reaches the source state in the corresponding machine. This holds in general:
the following proposition shows that we can use the greedy function to test for
coverage. Note that this crucial property is specific to machine induced from
episodes. It will not hold for a general machine.

\begin{proposition}
\label{prop:greedy}
Let $G$ be an episode, then a sequence $s$ covers $X$, a state in $\simple{M_G}$, if and only if 
$\greedy{X, s} = \set{i}$, the source state of $\simple{M_G}$.
\end{proposition}

\begin{corollary}
Let $G$ be an episode and let $X$ be the sink state of $\simple{M_G}$.
A sequence $s$ covers $G$ if and only if $\greedy{X, s} = \set{i}$, the source state of $\simple{M_G}$.
\end{corollary}

\subsection{Machine recognising minimal windows}
\label{sec:minmach}

So far we have constructed $M_G$ and $\simple{M_G}$ that recognise when
a sequence covers $G$. However, we are interested in finding out when
a sequence is a minimal window for $G$.

Assume that we are given an episode $G$ and let $M = \simple{M_G}$. Let $I =
\set{i}$ be the source state of $M$ and let $S = \set{x_G}$ be the sink state
of $M$.  We define two machines,
\begin{enumerate}
\item $M_1$ is obtained from $M$ by adding a new
source state, say $J$, and adding an edge $(J, I)$ for each possible label.
\item $M_2$ is obtained from $M$ by adding a new
sink state, say $T$ and adding an edge $(S, T)$ for each possible label.
\end{enumerate}
Both $M_1$ and $M_2$ are simple.

Let us first consider $M_1$.  Assume that we are given a sequence $s =
s_1\cdots s_L$ such that $\greedy{S, s} = I$. Then we know immediately that $s$
covers $G$ but $s[2, L]$ does not. Now let us consider $M_2$. Sequence $s[1, L
- 1]$ covers $G$ if and only $\greedy{T, s} = I$. Consequently, we need to
design a machine that simultaneously computes $\greedy{S, s}$ for $M_1$
and $\greedy{T, s}$ for $M_2$.

In order to do so we need to define a special machine.  Assume that we are
given two simple machines $M_1$ and $M_2$, and a set of pairs of states $\Theta
= \set{(x_i, y_i)}_{i = 1}^N$, where $x_i$ is a state in $M_1$ and $y_i$ is a
state in $M_2$.
We will now define a join machine, $M^* = \co{M_1, M_2, \Theta}$, that is guaranteed to contain the states from $\Theta$.
To define the states of this machine, let $z_1$ be a
state in $M_1$ and let $z_2$ be a state in $M_2$. We first define a set of pairs of states recursively
\[
	f(z_1, z_2) = (z_1, z_2) \cup \bigcup_{a \in \inc{z_1} \cup \inc{z_2}} f(\greedy{z_1, a}, \greedy{z_2, a}).
\]
We define the states of $M^*$ to be $\bigcup_{\theta \in \Theta} f(\theta)$.
Two states $\alpha = (y_1, y_2)$ and $\beta = (z_1, z_2)$ are connected with an
edge $(\alpha, \beta)$ if and only if $y_i = \greedy{z_i, a}$ and $a \in
\inc{z_1} \cup \inc{z_2}$.  It follows immediately that $M^*$ is
simple.

\begin{proposition}
\label{prop:join}
Let $M_1$ and $M_2$ be two simple machines. Let $\Theta$ be a set of pairs of states.
Define $M^* = \co{M_1, M_2, \Theta}$. Let $\alpha = (x_1, x_2)$ be 
a state in $M^*$. Then
$\greedy{\alpha, s} = (\greedy{x_1, s}, \greedy{x_2, s})$.
\end{proposition}

We can now define a machine that we will use to test whether sequence is a
minimal window of $G$. Let $M_1$, $M_2$, $S$ and $T$ as defined above.
Let $M^* = \co{M_1, M_2, \set{(S, T)}}$.  The following proposition
demonstrates how we can use $M^*$ to characterise the minimal window. 

\begin{proposition}
\label{prop:minmach}
Let $M_1$, $M_2$, $M^*$, and $I$ be as defined above.
Let $\alpha$ be a sink state of $M^*$.
Then, a sequence $s$ is a minimal window for $G$ if and only if $\greedy{\alpha, s} \in \Omega$,
where $\Omega = \set{(I, Y) \mid I \neq Y \text{ is a state of } M_2 }$.
\end{proposition}

For the purpose of recognising minimal windows, there are lot of
redundant states in $M^*$. Any state that is not a child or part of
$\Omega$ can be removed and the outgoing edges reattached to the source state
without effecting the validity of Proposition~\ref{prop:minmach}. This is true
because once the greedy function reaches any such state then it will never
reach $\Omega$. To optimise we remove two types of non-source states: any of
form $(J, Y)$, where $J$ is the source state of $M_1$ and any state of form
$(Y, Y)$. We refer to the resulting machine as $\mwmach{G}$.

\begin{example}
\label{ex:join}
Consider an episode $G$ given in Figure~\ref{fig:join:a}.  The machine
$\simple{M_G}$ is given in Figure~\ref{fig:join:b} and the augmented versions
$M_1$ and $M_2$ are given in Figures~\ref{fig:join:c}--\ref{fig:join:d}.
These machines are then combined to $M^* = \co{M_1, M_2, \set{(x_1, x_0)}}$, given in Figure~\ref{fig:join:e}.

The final, simplified, machine is given in Figure~\ref{fig:join:f}. In order to a sequence
to be a minimal window for $G$, the greedy function must land either in $(x_5,
x_3)$, $(x_5, x_1)$, or in $(x_5, x_4)$. Note that many states from $M^*$
are removed. For example, if we are in $x_1x_0$ and we see any other symbol
than $c$, then we know that $s$ is not a minimal window since $s$ must end with $c$
in order to be one. 

\begin{figure}[htb!]
\begin{center}
\begin{minipage}[b]{3cm}
\subcaptionbox{$G$\label{fig:join:a}}{
\begin{tikzpicture}
\node[ep] (n1) at (0, 0.4) {$a$};
\node[ep] (n2) at (0, -0.4) {$b$};
\node[ep] (n3) at (0.8, 0) {$c$};
\draw[sedge] (n1) -> (n3);
\draw[sedge] (n2) -> (n3);
\end{tikzpicture}

}\hfill%
\subcaptionbox{$\simple{M_G}$\label{fig:join:b}}[1.6cm]{
\begin{tikzpicture}
\node[state] (n1) at (0, 0) {$x_1$};
\node[state] (n2) at (0, -0.8) {$x_2$};
\node[state] (n3) at (-0.4, -1.6) {$x_3$};
\node[state] (n4) at (0.4, -1.6) {$x_4$};
\node[state] (n5) at (0, -2.4) {$x_5$};
\draw[sedge] (n2) -> node[label, right] {$c$} (n1);
\draw[sedge] (n3) -> node[label, left] {$b$} (n2);
\draw[sedge] (n4) -> node[label, right] {$a$} (n2);
\draw[sedge] (n5) -> node[label, left] {$a$} (n3);
\draw[sedge] (n5) -> node[label, right] {$b$} (n4);
\end{tikzpicture}

}\\

\subcaptionbox{$M_1$\label{fig:join:c}}{
\begin{tikzpicture}
\node[state] (n1) at (0, 0) {$x_1$};
\node[state] (n2) at (0, -0.8) {$x_2$};
\node[state] (n3) at (-0.4, -1.6) {$x_3$};
\node[state] (n4) at (0.4, -1.6) {$x_4$};
\node[state] (n5) at (0, -2.4) {$x_5$};
\node[state] (n6) at (0, -3.2) {$x_6$};
\draw[sedge] (n2) -> node[label, right] {$c$} (n1);
\draw[sedge] (n3) -> node[label, left] {$b$} (n2);
\draw[sedge] (n4) -> node[label, right] {$a$} (n2);
\draw[sedge] (n5) -> node[label, left] {$a$} (n3);
\draw[sedge] (n5) -> node[label, right] {$b$} (n4);
\draw[sedge] (n6) -> (n5);
\end{tikzpicture}

}\hfill%
\subcaptionbox{$M_2$\label{fig:join:d}}{
\begin{tikzpicture}
\node[state] (n0) at (0, 0.8) {$x_0$};
\node[state] (n1) at (0, 0) {$x_1$};
\node[state] (n2) at (0, -0.8) {$x_2$};
\node[state] (n3) at (-0.4, -1.6) {$x_3$};
\node[state] (n4) at (0.4, -1.6) {$x_4$};
\node[state] (n5) at (0, -2.4) {$x_5$};
\draw[sedge] (n2) -> node[label, right] {$c$} (n1);
\draw[sedge] (n3) -> node[label, left] {$b$} (n2);
\draw[sedge] (n4) -> node[label, right] {$a$} (n2);
\draw[sedge] (n5) -> node[label, left] {$a$} (n3);
\draw[sedge] (n5) -> node[label, right] {$b$} (n4);
\draw[sedge] (n1) -> (n0);
\end{tikzpicture}

}
\end{minipage}%
\subcaptionbox{$\co{M_1, M_2, \set{(x_1, x_0)}}$\label{fig:join:e}}{
\begin{tikzpicture}
\node[state2, anchor = south] (n53) at (-1.1, 0.4) {$x_5x_3$};
\node[state, anchor = south] (n42) at (-1, 1.2) {$x_4x_2$};
\node[state2, anchor = south] (n51) at (0, 1.2) {$x_5x_1$};
\node[state, anchor = south] (n32) at (1, 1.2) {$x_3x_2$};
\node[state2, anchor = south] (n54) at (1.1, 0.4) {$x_5x_4$};
\node[state, anchor = south] (n41) at (-0.5, 2) {$x_4x_1$};
\node[state, anchor = south] (n31) at (0.5, 2) {$x_3x_1$};
\node[state, anchor = south] (n21) at (0, 2.8) {$x_2x_1$};
\node[state, anchor = south] (n10) at (-1, 3.6) {$x_1x_0$};
\node[state, anchor = south] (n61) at (0.2, 0.4) {$x_6x_1$};
\node[state, anchor = south] (n62) at (0, -0.4) {$x_6x_2$};
\node[state, anchor = south] (n63) at (-0.5, -1.2) {$x_6x_3$};
\node[state, anchor = south] (n64) at (0.5, -1.2) {$x_6x_4$};
\node[state, anchor = south] (n65) at (0, -2) {$x_6x_5$};
\node[state, anchor = south] (n1) at (-2.4, 2.8) {$x_1x_1$};
\node[state, anchor = south] (n2) at (-2.4, 2) {$x_2x_2$};
\node[state, anchor = south] (n3) at (-2.8, 1.2) {$x_3x_3$};
\node[state, anchor = south] (n4) at (-2, 1.2) {$x_4x_4$};
\node[state, anchor = south] (n5) at (-2.4, 0.4) {$x_5x_5$};
\draw[sedge2] (n2) -> node[label, right] {$c$} (n1);
\draw[sedge2] (n3) -> node[label, left] {$b$} (n2);
\draw[sedge2] (n4) -> node[label, right] {$a$} (n2);
\draw[sedge2] (n5) -> node[label, left] {$a$} (n3);
\draw[sedge2] (n5) -> node[label, right] {$b$} (n4);
\draw[sedge] (n54) -> node[label, right] {$a$} (n32);
\draw[sedge] (n53) -> node[label, left] {$b$} (n42);
\draw[sedge] (n42) -> node[label, right] {$c$} (n41);
\draw[sedge] (n51) -> node[label, right] {$b$} (n41);
\draw[sedge] (n51) -> node[label, right] {$a$} (n31);
\draw[sedge] (n32) -> node[label, right] {$c$} (n31);
\draw[sedge] (n41) -> node[label, left] {$a$} (n21);
\draw[sedge] (n31) -> node[label, left] {$b$} (n21);
\draw[sedge] (n21) -> node[label, right] {$c$} (n10);
\draw[sedge2] (n61) -> node[label, right] {$a, b$} (n51);
\draw[sedge2, bend left = 45] (n62) edge node[label, left] {$c$} (n51);
\draw[sedge2] (n62) -> node[label, right] {$c$} (n61);
\draw[sedge2] (n63) -> node[label, right] {$b$} (n62);
\draw[sedge2] (n64) -> node[label, right] {$a$} (n62);
\draw[sedge2] (n65) -> node[label, right] {$a$} (n63);
\draw[sedge2] (n65) -> node[label, right] {$b$} (n64);
\draw[sedge2] (n64) -> node[label, right, inner sep = 1pt] {$a, c$} (n54);
\draw[sedge2, bend right = 45] (n65) edge node[label, right] {$b$} (n54);
\draw[sedge2] (n63) -> node[label, right, inner sep = 1pt, near end] {$b, c$} (n53);
\draw[sedge2, bend left = 45] (n65) edge node[label, right] {$a$} (n53);
\draw[sedge2] (n2) -> node[label, below] {$c$} (n21);
\draw[sedge2] (n4) -> node[label, below] {$a$} (n42);
\draw[sedge2, bend left = 10] (n3) edge node[label, above, near start] {$b$} (n32);
\draw[sedge2] (n1) -> node[label, right] {$a, b$} (n10);
\draw[sedge2, bend left = 45] (n65) edge node[label, right] {} (n5);
\end{tikzpicture}}\hspace*{-0.51cm}%
\subcaptionbox{$\mwmach{G}$\label{fig:join:f}}{
\begin{tikzpicture}
\node[state, anchor = south] (i) at (0, 0) {$\psi$};
\node[state2, anchor = south] (n53) at (-1, 0.4) {$x_5x_3$};
\node[state, anchor = south] (n42) at (-1, 1.2) {$x_4x_2$};
\node[state2, anchor = south] (n51) at (0, 1.2) {$x_5x_1$};
\node[state, anchor = south] (n32) at (1, 1.2) {$x_3x_2$};
\node[state2, anchor = south] (n54) at (1, 0.4) {$x_5x_4$};
\node[state, anchor = south] (n41) at (-0.5, 2) {$x_4x_1$};
\node[state, anchor = south] (n31) at (0.5, 2) {$x_3x_1$};
\node[state, anchor = south] (n21) at (0, 2.8) {$x_2x_1$};
\node[state, anchor = south] (n10) at (-1, 3.6) {$x_1x_0$};
\draw[sedge] (i) -> (n53);
\draw[sedge] (i) -> node[label, right, near end] {$a$} (n42);
\draw[sedge] (i) -> (n51);
\draw[sedge] (i) -> node[label, left, near end] {$b$} (n32);
\draw[sedge] (i) -> (n54);
\draw[sedge] (n54) -> node[label, right] {$a$} (n32);
\draw[sedge] (n53) -> node[label, left] {$b$} (n42);
\draw[sedge] (n42) -> node[label, right] {$c$} (n41);
\draw[sedge] (n51) -> node[label, right] {$b$} (n41);
\draw[sedge] (n51) -> node[label, right] {$a$} (n31);
\draw[sedge] (n32) -> node[label, right] {$c$} (n31);
\draw[sedge] (n41) -> node[label, left] {$a$} (n21);
\draw[sedge] (n31) -> node[label, left] {$b$} (n21);
\draw[sedge] (n21) -> node[label, right] {$c$} (n10);
\draw[sedge] (i) .. controls (1.8, 0) and (2.8, 1) .. node[label, left, pos = 0.6] {$c$} (n21);
\draw[sedge] (i) .. controls (-1.8, 0) and (-2.4, 1) .. node[label, right, very near end] {$a, b$} (n10);
\end{tikzpicture}}%
\end{center}
\caption{Toy episode and related machines. Figure~\ref{fig:join:a} contains an episode $G$. A simple machine $\simple{M_G}$ is given in Figure~\ref{fig:join:b}.
Machines given in Figures~\ref{fig:join:c}--\ref{fig:join:d} are used to construct a machine to recognise a minimal window, Figure~\ref{fig:join:e}.
In order to a sequence to be a minimal window we must land to a highlighted node when starting from $x_1x_0$.
The states with dashed outgoing edges are redundant and can be be collapsed, resulting in a machine
given in Figure~\ref{fig:join:f}.}
\label{fig:join}
\end{figure}

\end{example}

\section{Computing Moments}
\label{sec:moments}
Now that we have defined a machine for recognising a minimal window,
we will use it to compute the needed probabilities.
In Section~\ref{sec:prob} we demonstrate how to use the machine
to compute the expected weight. In Section~\ref{sec:normal}
we show the asymptotic normality and in Section~\ref{sec:cross} we demonstrate
how to compute the variance.
We finish the section by considering computational complexity.

\subsection{Computing probabilities}
\label{sec:prob}

Proposition~\ref{prop:minmach} gives us means to express the minimal window
using a machine and the greedy function. In this section we demonstrate
how to compute probabilities that the greedy function lands in some particular
state.

Let $M$ be a simple machine. Let $Y$ be a set of states in $M$ and let $x$ be a
state in $M$. Let us first define
\[
	\pgreedy{x, Y, L} =  p(\greedy{x, s} \in Y \mid \abs{s} = L)
\]
to be the probability that a random sequence $s$ of length $L$ reaches one of the
states in $Y$.

\begin{proposition}
\label{prop:greedyprob}
Let $M$ be a simple machine. Let $Y$ be a set of states in $M$ and let $x$ be a
state in $M$.

Then it holds that for $L > 0$,
\begin{equation}
\label{eq:greedyprob}
	\pgreedy{x, Y, L}  = \sum_{a \in \Sigma} p(a) \pgreedy{\greedy{x, a},  Y,  L - 1}.
\end{equation}
For $L = 0$, we have
\[
	\pgreedy{x, Y, 0} =
	\begin{cases}
		1 & \text{if } x \in Y, \\
		0 & \text{if } x \notin Y.
	\end{cases}
\]
\end{proposition}

\begin{example}
Consider a machine $\mwmach{G}$ given Figure~\ref{fig:join:f}.
Assume that the individual probabilities are $p(a) = 0.3$, $p(b) = 0.2$, and $p(c) = 0.5$.
The according to Proposition~\ref{eq:greedyprob}, $\pgreedy{x_4x_2, x_5x_3, 1} = 0.2$
and
\[
	\pgreedy{x_4x_2, x_5x_3, L} = 0.5\pgreedy{x_4x_2, x_5x_3, L - 1}
\]
for $L > 1$, which implies that $\pgreedy{x_4x_2, x_5x_3, L} = 0.2\times0.5^{L - 1}$.
We can verify this by observing
that the sequence of $L$ events that leads from $x_4x_2$ to $x_5x_3$ must have $L - 1$
events labelled as $c$ followed by one $b$.
\end{example}

To solve the needed quantities, we need to compute moments,
\[
	\moment{x, f, Y} = \sum_{L = 1}^\infty f(L) \pgreedy{x, Y, L}.
\]

Proposition~\ref{prop:minmach} now immediately implies that we can express the
needed statistics using moments.

\begin{proposition}
\label{prop:epimoment}
Assume an episode $G$.
Let $M = \mwmach{G}$ and let $\alpha$ and $\Omega$ be as in Proposition~\ref{prop:minmach}.
Let $Y_i$, $X_i$ and $Z_i$ be defined as in Section~\ref{sec:window}.
Then
\[
\begin{split}
	\mean{X_1} &= \moment{\alpha, f, \Omega}, \quad\text{for}\quad f(L) = 1, \\
	\mean{Y_1} &= \moment{\alpha, f, \Omega}, \quad\text{for}\quad f(L) = L, \\
	\mean{Z_1} &= \moment{\alpha, f, \Omega}, \quad\text{for}\quad f(L) = \rho^L, \\
	\mean{Z_1^2} &= \moment{\alpha, f, \Omega}, \quad\text{for}\quad f(L) = \rho^{2L}, \\
	\mean{Y_1Z_1} &= \moment{\alpha, f, \Omega}, \quad\text{for}\quad f(L) = \rho^LL. \\
\end{split}
\]
\end{proposition}

Note that the sum has infinite number of terms, hence we cannot compute this
by raw application of Proposition~\ref{prop:greedyprob}. Luckily, we can express
moments in closed recursive form. First, we need to show that the moments
we consider are finite.

\begin{lemma}
\label{lem:finite}
Let $M$ be a simple machine. Let $Y$ be a set of states in $M$. Assume that $p(a) > 0$ for all $a \in \Sigma$.
Assume that we are given a function $f$ such that $f(L)$ grows at polynomial rate.
If the source node is not contained in $Y$, then $m(x, f, Y)$ is finite for any state $x$.
\end{lemma}

\begin{proposition}
\label{prop:greedymom}
Let $M$ be a simple machine. Assume that we have a function $f$ mapping an integer to a real number.
Assume also for $L \geq 1$, we have $f(L - 1) = cf(L) + h(L)$ for some $c \in \real$ and a function $h$.
Assume that $f$ and $g$ grow at polynomial rate, at maximum.
Let $q = 1 - \sum_{a \in \inc{x}} p(a)$ and set $r = c - q$.
Let $i(y) = \pgreedy{y, Y, 0}f(0)$.
Then
\[
	\moment{x, f, Y} = \frac{1}{r}\big(qi(x) - \moment{x, h, Y} + \sum_{a \in \inc{x} \atop y = \greedy{x, a}} p(a)(\moment{y, f, Y} + i(y))\big).
\]
\end{proposition}

We can now use Proposition~\ref{prop:greedymom} to compute the moments given in Proposition~\ref{prop:epimoment}.

\begin{proposition}
\label{prop:recurse}
The identity $f(L - 1) = cf(L) + h(L)$ holds for the following functions,
\[
\begin{split}
	f(L) = 1, &\quad\text{for}\quad c = 1, \ h(L) = 0, \\
	f(L) = L, &\quad\text{for}\quad c = 1, \ h(L) = -1, \\
	f(L) = \rho^L, &\quad\text{for}\quad c = \rho^{-1}, \ h(L) = 0, \\
	f(L) = \rho^{2L}, &\quad\text{for}\quad c = \rho^{-2}, \ h(L) = 0, \\
	f(L) = \rho^{L}L, &\quad\text{for}\quad c = \rho^{-1}, \ h(L) = -\rho^{L - 1}.
\end{split}
\]
\end{proposition}

\begin{example}
Consider machine $\mwmach{G}$ given in Figure~\ref{fig:join:e}. Let us
define $\Omega = \set{(x_5, x_3), (x_5, x_1), (x_5, x_4)}$. Assume also that
the probabilities for the symbols are $p(a) = 0.3$, $p(b) = 0.2$, and $p(c) =
0.5$. Let $f(L) = 1$.

Then using Proposition~\ref{prop:greedymom} we see that
\begin{align*}
	\moment{(x_4, x_2), f, \Omega} & = 0.2/0.5 = 0.4, \\
	\moment{(x_3, x_2), f, \Omega} & = 0.3/0.5 = 0.6, \\
	\moment{(x_4, x_1), f, \Omega} & = (0.2 + 0.5\times 0.4)/0.7 = 4/7, \\ 
	\moment{(x_3, x_1), f, \Omega} & = (0.3 + 0.5\times 0.6)/0.8 = 3/4,\\
	\moment{(x_2, x_1), f, \Omega} & = 0.3\times 4/7 + 0.2\times 3/4 = 0.32, \\
	\moment{(x_1, x_1), f, \Omega} & = 0.5\times 0.32 = 0.16,
\end{align*}
and the moment for the remaining states is equal to $0$.
\end{example}

Proposition~\ref{prop:greedymom} gives us means for a straightforward algorithm
\textsc{Moments} for computing moments (given in Algorithm~\ref{alg:moments}).
\textsc{Moments} takes as input a simple machine $M$, a map $i$ for initial
values, a map $h$ for update values, and a constant $c$.
Note that \textsc{Moments} is linear function of $i$ and $h$, that is,
\[
\begin{split}
	&\textsc{Moments}(M, k_1i_1 + k_2i_2, k_1h_1 + k_2h_2, c) =  \\
	&\quad k_1\textsc{Moments}(M, i_1, h_1, c) + k_2\textsc{Moments}(M, i_2, h_2, c)
\end{split}
\]
for any constants $k_1$ and $k_2$. We will use this property later for speed-ups.

\begin{algorithm}
\Input{a simple machine $M$, a map $i$ for initial values, a map $h$ for update values, and a constant $c$ for recursive update}
\Output{Moment of $f$ for every state $x \in M$}
	\For{ $x \in M$ in topological order} {
		$q \define  1 - \sum_{a \in \inc{x}} p(a)$\;
		$r \define c - q$\;
		$m(x) \define \frac{1}{r}\big(q i(x) - h(x) + \sum\limits_{a \in \inc{x} \atop y = \greedy{x, a}} p(a)(m(y) + i(y))\big)$\;
	}
\Return $m$\;
\caption{\textsc{Moments}$(M, i, h, c)$ computes moments using Proposition~\ref{prop:greedymom}.}
\label{alg:moments}
\end{algorithm}

\subsection{Asymptotic Normality}
\label{sec:normal}

We will now prove that our statistic approaches to
the normal distribution. The proof is not trivial since the variables $X_i$ and
$Z_i$ are not independent. Hence we will use Central Limit Theorem for strongly
mixing sequences.

Our first step is to show that the sequence the central limit theorem holds for $(Z_i, X_i)$. 

\begin{proposition}
\label{prop:normal}
Let $G$ be an episode.
Sequence $1/\sqrt{L}\sum_{k = 1}^L (Z_k, X_k) - (q, p)$ converges in distribution to $N(0, C)$, where
$q = \mean{Z_1}$, $p = \mean{X_1}$, and $C$ is a $2 \times 2$ covariance matrix,
$C_{11}  =  \var{Z_1} + 2D_{11}$, $C_{22} = \var{X_1} + 2D_{22}$,  $C_{21} = C_{12}  =   \cov{X_1, Z_1} + D_{12} + D_{21}$,  where
\begin{align*}
	D_{11}  = &  \sum_{i = 2}^\infty \mean{(Z_1 - q)(Z_i - q)}, &  D_{22}  = &  \sum_{i = 2}^\infty \mean{(X_1 - p)(X_i - p)}, \\
	D_{12}  = &  \sum_{i = 2}^\infty \mean{(Z_1 - q)(X_i - p)}, &  D_{21}  = &  \sum_{i = 2}^\infty \mean{(X_1 - p)(Z_i - q)}. \\
\end{align*}
\end{proposition}

Since the central limit theorem holds for $(Z_i, X_i)$, we can apply this to obtain the main result.

\begin{proposition}
\label{prop:rationormal}
Let $G$ be an episode.
Let $p$, $q$ and $C$ be as in Proposition~\ref{prop:normal}. Define $\mu = q/p$.
Then
\[
	\sqrt{L}\Big(\frac{\sum_{k = 1}^L Z_k}{\sum_{k = 1}^L X_k} - \mu\Big)
\]
converges to $N(0, \sigma^2)$ as $L \to \infty$ , where $\sigma^2 = p^{-2}\pr{C_{11} -2\mu C_{12} + \mu^2C_{22}}$.
\end{proposition}

These results suggest that we can use $\Phi(-\score{G})$ as a $P$-value, where $\Phi$ is the cumulative density function
of the normal distribution. However, in practice we have several problems:
\begin{itemize}
\item The result is accurate only asymptotically. Moreover, the distribution of
$\score{G}$ can be heavily skewed so we need a large number of samples in order to estimate become accurate.

\item We do not have directly, the probabilities of individual items, instead
we will estimate the probabilities from the training sequence. This will introduce
some error in prediction making the $P$-values smaller than they should be.

\item We are computing a large number of statistical tests. In such case, it is
advisable to use some technique, for example, Bonferroni correction, to
compensate for the multiple hypotheses problem. However, it is not obvious
which technique should we use.
\end{itemize}

Because of these problems, instead of interpreting $\Phi(-\score{G})$ as a
$P$-value, we simply use $\score{G}$ to rank patterns and use it as a top-$K$
method.  Note that $\Phi$ is a monotonic function, hence the larger the score,
the smaller the $P$-value.

By studying the formulas in the above propositions we see that we can compute
the necessary statistics $p$ and $q$ using Proposition~\ref{prop:epimoment},
and consequently we can compute $\mu$. However, in order to compute the
variance $\sigma$ we need to compute $D_{11}$, $D_{12}$, $D_{21}$, and $D_{22}$
given in Proposition~\ref{prop:normal}. We will demonstrate a technique for
computing these statistics in the next section.

\subsection{Computing Cross-moments}
\label{sec:cross}
Our final step is to compute cross-moments given in
Proposition~\ref{prop:normal}.
In order to do so we first
need to prove a different formulation of these statistics. This formulation is
more fruitful as we no longer have to deal infinite sums. 

\begin{proposition}
\label{prop:cross}
Let $p$, $q$, $D_{11}$, $D_{12}$, $D_{21}$, and $D_{22}$ be as in Proposition~\ref{prop:normal}.
Define $v = \mean{Y_1}$ and  $w = \mean{Y_1Z_1}$.
Then
\begin{align*}
	D_{22} & =   \means{X_1\sum_{k = 2}^{Y_1}X_k} - (v - p)p,
	& D_{12} & =   \means{X_1Z_1\sum_{k = 2}^{Y_1}X_k} - (w - q)p, \\
	D_{21} & =  \means{X_1\sum_{k = 2}^{Y_1}Z_kX_k} - (v - p)q, 
	& D_{11} & =  \means{X_1Z_1\sum_{k = 2}^{Y_1}Z_kX_k} - (w - q)q. \\
\end{align*}
\end{proposition}

Our next step is to compute the moments.
To that end, let $M = \mwmach{G}$ be a
machine recognising the minimal window of $G$, let $\alpha$ be a sink state in $M$,
and let $\Omega$ be the states as in Proposition~\ref{prop:minmach}.  We will
study the probability $p(Y_1 = a, Y_k = a + b)$, where $a \geq k > 1$ and $b
\geq 1$. Let $u = s[k, a]$ and $v = s[a + 1, a + b]$.  The idea is to break the
probability into a sum of probabilities based on the state $\greedy{\alpha, v}$
and $\greedy{\alpha, u}$. These probabilities can be further decomposed into three
factors which we can then turn into moments using Proposition~\ref{prop:greedymom}.

Define a random variable $E = \greedy{\alpha, s[1, a]} \in \Omega$. This variable is true
if and only if $Y_1 = a$. In addition, define $F = \greedy{\alpha, s[k, a + b]}
\in \Omega$ and $G_\beta = \greedy{\beta, u} \in \Omega$.

Let us write $\Theta$ to be all proper intermediate states of $M$ between $\alpha$ and $\Omega$. 
Since $k > 1$, $Y_1 = a$ implies that $\greedy{\alpha, u} \in \Theta$. Similarly, $Y_k = a + b$
implies that $\greedy{\alpha, v} \in \Theta$. We can now write $p(Y_1 = a, Y_k = a + b)$ as

\begin{equation}
\label{eq:decompose}
\begin{split}
	&p(Y_1 = a, Y_k = a + b)  = p(E, F) \\
	&\quad=\sum_{\beta \in \Theta} p(E, F, \greedy{\alpha, v} = \beta)
	=\sum_{\beta \in \Theta} p(E, G_\beta, \greedy{\alpha, v} = \beta) \\
	&\quad=\sum_{\beta \in \Theta} p(E, G_\beta)p(\greedy{\alpha, v} = \beta) 
	=\sum_{\beta \in \Theta}\pgreedy{\alpha, \beta, b} p(E, G_\beta) \\
	&\quad=\sum_{\beta \in \Theta}\pgreedy{\alpha, \beta, b} \sum_{\gamma  \in \Theta}p(E, G_\beta, \greedy{\alpha, u} = \gamma) \\
	&\quad=\sum_{\beta \in \Theta}\pgreedy{\alpha, \beta, b} \sum_{\gamma  \in \Theta}p(\greedy{\gamma, s[1, k - 1]} \in \Omega, G_\beta, \greedy{\alpha, u} = \gamma) \\
	&\quad=\sum_{\beta \in \Theta}\pgreedy{\alpha, \beta, b} \sum_{\gamma  \in \Theta}\pgreedy{\gamma, \Omega, k - 1} p(G_\beta, \greedy{\alpha, u} = \gamma). \\
\end{split}
\end{equation}

The only non-trivial factor in Equation~\ref{eq:decompose} that we cannot solve using $M$  is
$p(\greedy{\beta, u} \in \Omega, \greedy{\alpha, u} = \gamma)$. To solve this we construct
yet another machine. Let $M^* = \co{M, M, \set{(\theta, \alpha) \mid \theta \in \Theta}}$
and let $\Omega^*_\gamma = \set{(\omega, \gamma) \mid \omega \in \Omega}$. Then Proposition~\ref{prop:join} implies that
\[
	p(G_\beta, \greedy{\alpha, u} = \gamma) = \pgreedy{(\beta, \alpha), \Omega^*_\gamma, a - k + 1}.
\]
This leads to
\begin{equation}
\label{eq:decompose2}
	p(Y_1 = a, Y_k = a + b) =\sum_{\beta, \gamma}\pgreedy{\alpha, \beta, b} \pgreedy{\gamma, \Omega, k - 1} \pgreedy{(\beta, \alpha), \Omega^*_\gamma, a - k + 1}.
\end{equation}

Let us write $A_k = Y_1 - k + 1$.  We now define a function $f$ by which we can express the missing cross-moments,
\[
	f(P, Q, R)  = \means{X_1\sum_{k = 2}^{Y_1}\rho^{P(k - 1) + QA_k + R(Y_k - Y_1)}X_k}.
\]

This function is particularly useful since we can now apply
Equation~\ref{eq:decompose2} and obtain a closed form using moments,

\begin{equation}
\label{eq:fmom}
\begin{split}
	f(P, Q, R) & = \sum_{k = 2}^\infty \sum_{b = 1}^\infty \sum_{a = k}^\infty \rho^{P(k - 1) + Q(a - k + 1) + Rb} p(Y_1 = a, Y_k = a + b) \\
	           & = \sum_{\beta, \gamma}\moment{\alpha, R, \beta} \moment{\gamma, P, \Omega} \moment{(\beta, \alpha), Q, \Omega^*_\gamma}.\\
\end{split}
\end{equation}

Let us now express the cross-moments using $f$. We see immediately that,
\[
\begin{split}
	\means{X_1\sum_{k = 2}^{Y_1}X_k} & = f(0, 0, 0), \\
	\means{X_1\sum_{k = 2}^{Y_1}Z_1X_k} & = \means{X_1\sum_{k = 2}^{Y_1}\rho^{(k - 1) + A_k }X_k} = f(1, 1, 0), \\
	\means{X_1\sum_{k = 2}^{Y_1}Z_kX_k} & = \means{X_1\sum_{k = 2}^{Y_1}\rho^{A_k + Y_k - Y_1}A_kX_k} = f(0, 1, 1), \\
	\means{X_1\sum_{k = 2}^{Y_1}Z_1X_kZ_k}  & = \means{X_1\sum_{k = 2}^{Y_1}X_k\rho^{(k - 1) + 2A_k + Y_k - Y_1}} = f(1, 2, 1). \\
\end{split}
\]

As a final step we describe how we can optimise computation of $f(P, Q, R)$.
First recall that \textsc{Moments}, given in Algorithm~\ref{alg:moments}, is
linear with respect to its parameters $i$ and $h$.
Consider Equation~\ref{eq:fmom}.  Instead of
computing the sum over $\beta$ explicitly, we can compute $\textsc{Moments}(M, i, 0, \rho^R)$,
where $i(\beta)$ is defined as $\sum_{\gamma}\moment{\gamma, \rho^P, \Omega} \moment{(\beta,
\alpha), \rho^Q, \Omega^*}$.  We can repeat this trick again to remove the explicit
sum over $\gamma$.  The pseudo-code taking into account these optimisations is
given in Algorithm~\ref{alg:computef}.

\begin{algorithm}
	$M \define \mwmach{G}$\;
	$\alpha \define$ sink state of $M$\;
	$\Omega \define$ as in defined in Proposition~\ref{prop:minmach}\;
	$\Theta \define$ intermediate states of $\alpha$ and $\Omega$\;
	$M^* \define \co{M, M, \set{(\theta, \alpha) \mid \theta \in \Theta}}$\;
	$i_1(\theta) \define I(\theta \in \Omega)$\;
	$m \define  \textsc{Moments}(M, i_1, 0, \rho^P)$\;
	\ForEach {state $x$ in $M^*$} {
		\If {$x = (\omega, \theta)$, where $\omega \in \Omega$ and $\theta \in \Theta$} {
			$i_2(x) \define m(\theta)$\;
		}
	}
	$m \define  \textsc{Moments}(M^*, i_2, \rho^Q)$\;
	\ForEach {$\theta \in \Theta$} {
		$i_3(\theta) \define m((\theta, \alpha))$\;
	}
	$m \define  \textsc{Moments}(M, i_3,0, \rho^R)$\;
	\Return $m(\alpha)$\;
\caption{\textsc{CrossMoments}}
\label{alg:computef}
\end{algorithm}

\begin{example}
Let us compute $f(0, 1, 1)$ for an episode $G$ given in
Figure~\ref{fig:join:a}.  Let $M = \mwmach{G}$, given in
Figure~\ref{fig:cross}. Note that this machine is the same machine given in
Figure~\ref{fig:join:f}.  Let us define $\Omega = \set{\omega_1, \omega_2,
\omega_3}$. Assume also that the probabilities for the symbols are $p(a) =
0.3$, $p(b) = 0.2$, and $p(c) = 0.5$ and assume that we selected $\rho = 1/2$.
Define $h_k(x) = \rho^{kx}$.

To compute $f(0, 1, 1)$ we need to compute moments from three different
machines.  The obtained moments from a previous machine is fed as initial
values to the next machine as shown in Figure~\ref{fig:cross}. We use $M$ for the
first and the third machine. The second machine is $\co{M, M, \enset{(\theta_1, \alpha)}{(\theta_5, \alpha)}}$ with redundant states removed.
This machine is given in
Figure~\ref{fig:cross}.

\begin{figure}[htb!]
\begin{center}
\hspace*{-0.6cm}
\begin{tikzpicture}
\node[state, anchor = south] (i) at (-4, -1.2) {$\psi$};
\node[state2, anchor = south] (n53) at (-5, -0.8) {$\omega_1$};
\node[state, anchor = south] (n42) at (-5, 0) {$\theta_4$};
\node[state2, anchor = south] (n51) at (-4, 0) {$\omega_2$};
\node[state, anchor = south] (n32) at (-3, 0) {$\theta_5$};
\node[state2, anchor = south] (n54) at (-3, -0.8) {$\omega_3$};
\node[state, anchor = south] (n41) at (-4.5, 0.8) {$\theta_2$};
\node[state, anchor = south] (n31) at (-3.5, 0.8) {$\theta_3$};
\node[state, anchor = south] (n21) at (-4, 1.4) {$\theta_1$};
\node[state, anchor = south] (n10) at (-5, 2.4) {$\alpha$};
\draw[sedge] (i) -> (n53);
\draw[sedge] (i) -> node[label, right, near end] {$a$} (n42);
\draw[sedge] (i) -> (n51);
\draw[sedge] (i) -> node[label, left, near end] {$b$} (n32);
\draw[sedge] (i) -> (n54);
\draw[sedge] (n54) -> node[label, right] {$a$} (n32);
\draw[sedge] (n53) -> node[label, left] {$b$} (n42);
\draw[sedge] (n42) -> node[label, right] {$c$} (n41);
\draw[sedge] (n51) -> node[label, right] {$b$} (n41);
\draw[sedge] (n51) -> node[label, right] {$a$} (n31);
\draw[sedge] (n32) -> node[label, right] {$c$} (n31);
\draw[sedge] (n41) -> node[label, left] {$a$} (n21);
\draw[sedge] (n31) -> node[label, left] {$b$} (n21);
\draw[sedge] (n21) -> node[label, right] {$c$} (n10);
\draw[sedge] (i) .. controls (-2.2, -1.2) and (-1.2, -0.2) .. node[label, left, pos = 0.6] {$c$} (n21);
\draw[sedge] (i) .. controls (-5.8, -1.2) and (-6.4, -0.2) .. node[label, right, very near end] {$a, b$} (n10);

\node[state, anchor = south] (ii) at (0, -0.4) {$\psi$};
\node[state2, anchor = south] (o1) at (-1, 0.4) {$\omega_1\theta_3$};
\node[state, anchor = south] (t4t1) at (-1, 1.2) {$\theta_4\theta_1$};
\node[state, anchor = south] (t5t1) at (1, 1.2) {$\theta_5\theta_1$};
\node[state2, anchor = south] (o3) at (1, 0.4) {$\omega_3\theta_2$};
\node[state, anchor = south] (t2a) at (-0.5, 2) {$\theta_2\alpha$};
\node[state, anchor = south] (t3a) at (0.5, 2) {$\theta_3\alpha$};
\node[state, anchor = south] (t4a) at (-1.5, 2) {$\theta_4\alpha$};
\node[state, anchor = south] (t5a) at (1.5, 2) {$\theta_5\alpha$};
\node[state, anchor = south] (t1a) at (0, -1.2) {$\theta_1\alpha$};
\draw[sedge] (ii) -> (o1);
\draw[sedge] (ii) -> node[label, right, near end] {$a$} (t4t1);
\draw[sedge] (ii) -> node[label, left, near end] {$b$} (t5t1);
\draw[sedge] (ii) -> (o3);
\draw[sedge] (ii) -> (t1a);
\draw[sedge] (o3) -> node[label, right] {$a$} (t5t1);
\draw[sedge] (o1) -> node[label, left] {$b$} (t4t1);
\draw[sedge] (t4t1) -> node[label, right] {$c$} (t2a);
\draw[sedge] (t4t1) -> node[label, right] {$c$} (t4a);
\draw[sedge] (ii) -> node[label, right, near end, inner sep = 1pt] {$a, b$} (t2a);
\draw[sedge] (ii) -> node[label, left, very near end, inner sep = 1pt] {$a, b$} (t3a);
\draw[sedge] (t5t1) -> node[label, right] {$c$} (t3a);
\draw[sedge] (t5t1) -> node[label, right] {$c$} (t5a);
\draw[sedge] (ii) .. controls (-1.1, 0) and (-1.9, 0.5) .. node[label, left, very near end] {$c$} (t4t1);
\draw[sedge] (ii) .. controls (-1.3, -0.3) and (-2.2, 0.5) .. node[label, left, very near end] {$a, b$} (t4a);
\draw[sedge] (ii) .. controls (1.1, 0) and (1.9, 0.5) .. node[label, right, very near end] {$c$} (t5t1);
\draw[sedge] (ii) .. controls (1.3, -0.3) and (2.2, 0.5) .. node[label, right, very near end] {$a, b$} (t5a);

\node[state, anchor = south] (p) at (4, -1.2) {$\psi$};
\node[state, anchor = south] (m53) at (3, -0.8) {$\omega_1$};
\node[state2, anchor = south] (m42) at (3, 0) {$\theta_4$};
\node[state, anchor = south] (m51) at (4, 0) {$\omega_2$};
\node[state2, anchor = south] (m32) at (5, 0) {$\theta_5$};
\node[state, anchor = south] (m54) at (5, -0.8) {$\omega_3$};
\node[state2, anchor = south] (m41) at (3.5, 0.8) {$\theta_2$};
\node[state2, anchor = south] (m31) at (4.5, 0.8) {$\theta_3$};
\node[state2, anchor = south] (m21) at (4, 1.4) {$\theta_1$};
\node[state, anchor = south] (m10) at (3, 2.4) {$\alpha$};
\draw[sedge] (p) -> (m53);
\draw[sedge] (p) -> node[label, right, near end] {$a$} (m42);
\draw[sedge] (p) -> (m51);
\draw[sedge] (p) -> node[label, left, near end] {$b$} (m32);
\draw[sedge] (p) -> (m54);
\draw[sedge] (m54) -> node[label, right] {$a$} (m32);
\draw[sedge] (m53) -> node[label, left] {$b$} (m42);
\draw[sedge] (m42) -> node[label, right] {$c$} (m41);
\draw[sedge] (m51) -> node[label, right] {$b$} (m41);
\draw[sedge] (m51) -> node[label, right] {$a$} (m31);
\draw[sedge] (m32) -> node[label, right] {$c$} (m31);
\draw[sedge] (m41) -> node[label, left] {$a$} (m21);
\draw[sedge] (m31) -> node[label, left] {$b$} (m21);
\draw[sedge] (m21) -> node[label, right] {$c$} (m10);
\draw[sedge] (p) .. controls (5.8, -1.2) and (6.8, -0.2) .. node[label, left, pos = 0.6] {$c$} (m21);
\draw[sedge] (p) .. controls (2.2, -1.2) and (1.6, -0.2) .. node[label, right, very near end] {$a, b$} (m10);

\begin{pgfonlayer}{background}
\draw[sedge3, bend right = 20] (n41) edge (o3);
\draw[sedge3, bend right = 20] (n31) edge (o1);
\draw[sedge3, bend left = 7] (t1a) edge (m21);
\draw[sedge3, bend left = 30] (t2a) edge (m41);
\draw[sedge3, out = 20, in = 90] (t3a) edge (m31);
\draw[sedge3, out = -50, in = 190] (t4a) edge (m42);
\draw[sedge3, bend right = 20] (t5a) edge (m32);
\end{pgfonlayer}

\end{tikzpicture}%
\end{center}
\caption{Machines needed to compute the cross-moments for an episode $G$ given in Figure~\ref{fig:join:a}.
The first and the third machines are $M = \mwmach{G}$ and the second machine  is
$\co{M, M, \enset{(\theta_1, \alpha)}{(\theta_5, \alpha)}}$. We simplified the machine by collapsing
all states containing $\psi$ to one state. The arrows between the machines indicate how the moments
from the previous machines are passed to the next machine as initial values.}
\label{fig:cross}
\end{figure}
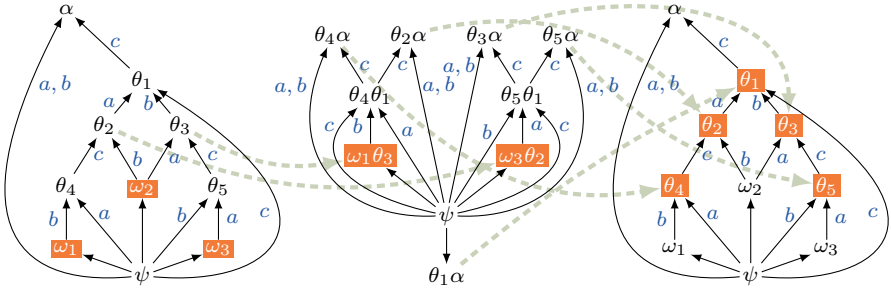

We start with $M$, and as initial values we set $1$ whenever a state
is in $\Omega$, and $0$ otherwise. This is equivalent to Example~\ref{ex:join}.
We need moments only for two states, $\theta_2$ and $\theta_3$, which are
\[
	\moment{\theta_2, h_0, \Omega} = 4/7 \quad\text{and}\quad \moment{\theta_3, h_0, \Omega} = 3/4.\\
\]

We now use the moments of $\theta_2$ and $\theta_3$ as initial values for $(\omega_3, \theta_2)$
and $(\omega_1, \theta_3)$, that is, we set $i_2((\omega_3, \theta_2)) = 4 / 7$ and
$i_2((\omega_1, \theta_3)) = 3/4$, and 0 for other states. We can now compute the moments,
\begin{align*}
	\moment{(\theta_4, \theta_1), h_1, i_2} & = (0.2 \times 3/4)/2 = 3 / 40, \\
	\moment{(\theta_5, \theta_1), h_1, i_2} & = (0.3 \times 4/7)/2 = 6 / 70, \\
	\moment{(\theta_2, \alpha), h_1, i_2} = \moment{(\theta_4, \alpha), h_1, i_2} & = (0.5 \times 3/40)/2 = 3 / 160, \\
	\moment{(\theta_3, \alpha), h_1, i_2} = \moment{(\theta_5, \alpha), h_1, i_2} & = (0.5 \times 6/70)/2 = 3 / 140, \\
\end{align*}
and $0$ for the remaining states. We feed these moments into initial values $i_3$
and compute the final moments,
\begin{align*}
	\moment{\theta_4, h_1, i_3} & = (0.5 \times 3 / 160) / 1.5 = 1 / 160,\\
	\moment{\theta_5, h_1, i_3} & = (0.5 \times 3 / 140) / 1.5 = 1 / 140,\\
	\moment{\theta_2, h_1, i_3} & = (0.3 \times 3 / 160 + 0.5 / 160) / 1.7 = 14 / 2720,  \\
	\moment{\theta_3, h_1, i_3} & = (0.2 \times 3 / 140 + 0.5 / 140) / 1.8 = 11 / 2520, \\
	\moment{\theta_1, h_1, i_3} & = (0.2 \times 14 / 2720  + 0.3 \times 11 / 2520) / 2 = 0.0012, \\
	\moment{\alpha, h_1, i_3} & = 0.5 \times 0.001 / 2 = 0.0003. \\
\end{align*}
Consequently, $f(0, 1, 1) =  \moment{\alpha, h_1, i_3} = 0.0003$.

\end{example}

\subsection{Computational complexity}
Let us now finish this section by discussing the computational complexity.
Given a machine $M$, evaluating moments will take $O(V(M) + E(M))$ time.
Hence, we need to study the sizes of our machines.  Given an episode $G$ with
$N$ nodes, the first machine $\mach{\efam{G}}$ may have $2^N$ states. This
happens if $G$ is a parallel episode. In practice, as we will see in the
experiments, this is not a problem since $N$ is typically small.

Exponentiality is (most likely) unavoidable since testing whether a sequence
covers an episode is known to be \NP-hard problem~\citep{tatti:11:mining}, and since we can
use $\mach{\efam{G}}$ to test coverage in polynomial time w.r.t.  the states in
$\mach{\efam{G}}$ we must have episodes for which we have exponential number of
states.

Simplifying $\mach{\efam{G}}$ may also lead to an exponential number of nodes.
This may happen if we have a lot of unrelated nodes with same labels. Typically,
this will not happen, especially, if the sequence has a large alphabet.
Moreover, we can avoid this problem by mining only strict episodes~\citep{tatti:12:mining}
in which we require that if there are two nodes with the same label, then one
of the nodes must be an ancestor of the other. For such episodes,
$\mach{\efam{G}}$ is already simple.

Computing a joint machine $\co{M, M}$ may result into a machine having $\abs{V(M)}^2$
states. In practice, the amount of states in $\mwmach{G}$ is much smaller since
not all pairs are considered. Similarly, a machine needed for computing cross-moments
may have $O(\abs{V(\mwmach{G})}^2)$ nodes. We will see that in our experiments the number
of states and edges remains small, making the method fast in practice.

\section{Related Work}
\label{sec:related}

Our approach can be seen as an extension of~\citep{tatti:09:significance} where
we developed a statistical test based on average length of minimal windows.  We
used a recursive update similar to the one given in
Proposition~\ref{prop:greedyprob}, however we capped the length of minimal
windows and computed explicitly the probabilities of an episode having a
minimal window of a certain length. In this work we avoid this by using
Proposition~\ref{prop:epimoment}. Additional limitation
of~\citep{tatti:09:significance} is that we were forced to simulate
cross-moments where in this work we compute them analytically.

Statistical measures for ranking episodes have been considered
by~\citet{gwadera:05:reliable,gwadera:05:markov} in which the authors
considered episode to be significant if the episode occurs too often or not
often enough in windows of fixed size.  As a background model the authors used
independence model in~\citep{gwadera:05:reliable} and Markov-chain model
in~\citep{gwadera:05:markov}. The authors' approach
in~\citep{gwadera:05:reliable} is similar to ours: First they construct a
finite state machine, essentially $\mach{G}$, and use recursive update similar
Proposition~\ref{prop:greedyprob} in order to compute the mean, that is, the
likelihood that the sequence of length $L$ covers the episode under
independence assumption.  The main difference between our approach and
theirs is that we base our measure directly on compactness, the average length of a minimal window,
while they base their measure on occurrence, that is, in how many windows the episode
occurs.

Working with the general episodes is difficult for two main reasons. Firstly,
general episodes are more prone to suffer from pattern explosion due to the
fact that there are so many directed acyclic graphs. Secondly, the simplest
task such as testing whether a sequence contains an episode is a \NP-hard
problem~\citep{tatti:11:mining}. Several subclasses of general episodes have
been suggested. \citet{pei:06:discovering} suggested mining episodes from set
of strings, sequences of unique symbols. \citet{tatti:12:mining} 
suggested discovering closed strict episodes. An episode is strict if two nodes
with the same label are always connected. \citet{achar:12:discovering} 
suggested discovering episodes with unique labels possibly with some
additional constraints, for example, the number of paths in a DAG.  The authors
suggested a score based on how evenly unconnected nodes occur in front of each
other. \citet{tatti:11:mining} considered a broader class of
episodes in which nodes are allowed to have multiple labels.

\citet{casas-garriga:03:discovering} proposed a criterion for
episodes by requiring that the consecutive symbols in a sequence should only
within a specified bound. While this approach attacks the problem of fixed
windows, it is still a frequency-based measure. This measure, however, is not
monotonic as it is pointed out by~\citet{meger:04:constraint-based}.  It would
be useful to see whether we can compute an expected value of this measure so
that we can compute a $P$-value based on some background model.

In a related work, \citet{cule:09:new} considered parallel episodes
significant if the smallest window containing each occurrence of a symbol of an
episode had a small value. Their approach differ from ours since the smallest
window containing a fixed occurrence of a symbol is not necessarily the minimal
window. Also, they consider only parallel episodes whereas we consider more
general DAG episodes. An interesting approach has been also taken
by~\citet{calders:07:mining} where the authors define a windowless frequency
measure of an itemset within a stream $s$ to be the frequency starting from a
certain point. This point is selected so that the frequency is maximal.
However, this method is defined for itemsets and it would be fruitful to see
whether this idea can be extended into episodes.

Finite state machines have been used
by~\citet{troncek:01:episode,hirao:01:practical} for discovering episodes.
However, their goal is different than ours since the actual machine is built
upon a sequence and not the episode set and it is used for discovering episodes
and not computing the coverage.

\section{Experiments}
In this section we present our experiments with the quality
measure using synthetic and real-world text sequences.
\label{sec:experiments}
\subsection{Datasets}
We conducted our experiments with several synthetic and real-world sequences. 

The first synthetic sequence, \emph{Ind} consists of $40\,000$ events drawn
independently and uniformly from an alphabet of $1\,000$ symbols.  The second
synthetic sequence, \emph{Plant} also contains $40\,000$ events independently and
uniformly from an alphabet of $1\,000$ symbols but in addition we planted 5 serial
episodes. Each episode consisted of 5 nodes, each node with a unique label.  We
planted each episode $100$ times and we added a gap between two consecutive
events with a $10\%$ probability.

Our third dataset, \emph{Moby}, is the novel Moby Dick by Herman
Melville.\!\footnote{The book was obtained from
\url{http://www.gutenberg.org/etext/15}.}  Our fourth sequence, \emph{Nsf}
consists of 739 first NSF award abstracts from 1990.\!\footnote{The abstracts
were obtained from \url{http://kdd.ics.uci.edu/databases/nsfabs/nsfawards.html}}
Our final dataset, \emph{Address}, consists of 
inaugural addresses of the presidents of the United States.\!\footnote{The
addresses were obtained from~\url{http://www.bartleby.com/124/pres68}.} To
avoid the historic concept drift---early speeches have different vocabulary than the later ones---we entwined the speeches by first taking the
odd ones and then even ones. Our fourth dataset, \emph{Jmlr}, consists of abstracts
from Journal of Machine Learning Research.\!\footnote{The abstracts were obtained from \url{http://jmlr.csail.mit.edu/}}
The sequences were processed using the Porter Stemmer and the stop words were removed.
The basic characteristics of sequences are summarised in Table~\ref{tab:basic}.

\begin{table}[htb!]
\caption{Characteristics of the sequences.  The second column contains the number of symbols in the
sequence. The third column contains the size of the alphabet of each sequence.}
\begin{center}
\begin{tabular}{lrr}
\toprule
Sequence & length & $\abs{\Sigma}$ \\
\midrule
\emph{Ind} & $40\,000$ & $1\,000$ \\
\emph{Plant} & $40\,000$ & $1\,000$ \\[1mm]
\emph{Moby} & $105\,719$ & $10\,277$ \\
\emph{Address} & $62\,066$ & $5\,295$ \\
\emph{Jmlr} & $75\,646$ & $3\,859$ \\
\emph{Nsf} & $35\,370$ & $4\,592$ \\
\bottomrule
\end{tabular}
\end{center}
\label{tab:basic}
\end{table}

\subsection{Experimental Setup}
Our experimental setup mimics the framework setup by~\citet{webb:07:discovering}
in which the data is divided into two parts, the first part is used for
discovering the patterns and the second part for testing whether the discovered
patterns were significant. We divided each sequence into two parts of
equivalent lengths. We used the first sequence for discovering the candidate
episodes and training the independence model. Then we tested the discovered episodes
against the model using the second sequence. We set parameter $\rho$ to $1/2$.

To generate candidate episodes we used a miner given by~\citet{tatti:12:mining}.
This miner discovers episodes in a breath-first fashion, that is, an episode is
tested if and only if all its sub-episodes are frequent. The miner outputs
closed\footnote{An episode is closed if there are no superepisode with the same
support.} and strict episodes. Requiring episodes to be closed reduces
redundancy between candidates considerably as there are typically many episodes
describing the same set of minimal windows.
The alphabet is large in our sequences, which implies
that it is quite unlikely to see the same symbol twice within a short window.
Consequently, there are only few non-strict frequent episodes.

As a constraint we required that the number of non-overlapping minimal windows
must exceed certain threshold in the first sequence.  This is a monotonic
condition that allows us to discover all candidates efficiently.  During mining
we also put an upper limit for minimal windows. The parameters and the numbers
of candidates are given in Table~\ref{tab:par}.

\begin{table}[htb!]
\caption{Parameters used for mining candidate episodes. The second column
contains the allowed maximal length of a minimal window during mining.  The
third column contains threshold for the number of disjoint minimal windows.
The fourth column contains the number of non-singleton episodes. }
\begin{center}
\begin{tabular}{lrrr}
\toprule
Sequence & max window & threshold & \# of episodes \\
\midrule
\emph{Ind} & $15$ & $4$ & $1\,249$ \\
\emph{Plant} & $15$ & $5$ & $734$ \\[1mm]
\emph{Moby} & $20$ & $10$ & $6\,043$ \\
\emph{Address} & $20$ & $4$ & $41\,888$ \\
\emph{Jmlr} & $20$ & $10$ & $14\,528$ \\
\emph{Nsf} & $20$ & $15$ & $2\,845$ \\
\bottomrule
\end{tabular}
\end{center}
\label{tab:par}
\end{table}

\subsection{Computational complexity}

Let us first study computational complexity in practice.  As we pointed out
earlier it is possible that sizes of structures needed to compute the score
become exponentially large. To demonstrate the sizes in practice we computed the average number of
states and edges in machines used to compute the score. The results are given
in Table~\ref{tab:sizes}.

\begin{table}[htb!]
\caption{Average sizes of machines used for ranking episodes. Even columns,
labelled with $\abs{V}$, contain the number of nodes, while odd columns,
labelled with $\abs{E}$, contain the number of edges. The first machine
$\mach{\efam{G}}$ recognises when episode is covered, the second machine
$\simple{\mach{\efam{G}}}$ is a simplification of $\mach{\efam{G}}$. The third
machine $\mwmach{G}$ tests whether a sequence is a minimal window, and the last
machine $M^*$ is used for computing cross-moments, see Section~\ref{sec:cross}. The last columns is the time needed to rank the discovered episodes per dataset.}
\begin{center}
\begin{tabular}{lrrcrrcrrcrrr}
\toprule
& \multicolumn{2}{l}{$\mach{\efam{G}}$} &
& \multicolumn{2}{l}{$\simple{\mach{\efam{G}}}$} &
& \multicolumn{2}{l}{$\mwmach{\efam{G}}$} &
& \multicolumn{2}{l}{$M^*$} \\
\cmidrule{2-3} \cmidrule{5-6} \cmidrule{8-9} \cmidrule{11-12}
Sequence & $\abs{V}$ & $\abs{E}$ && $\abs{V}$ & $\abs{E}$ && $\abs{V}$ & $\abs{E}$ && $\abs{V}$ & $\abs{E}$ & time (s) \\
\midrule
\emph{Ind} & $3.8$ & $3.7$ &  & $3.8$ & $3.7$ &  & $4.7$ & $3.7$ &  & $3.5$ & $1.7$ & $0.34$\\
\emph{Plant} & $4.4$ & $4.5$ &  & $4.4$ & $4.5$ &  & $6.6$ & $6.6$ &  & $11.5$ & $13.3$ & $0.28$\\[1mm]
\emph{Moby} & $3.9$ & $3.6$ &  & $3.9$ & $3.6$ &  & $4.7$ & $3.9$ &  & $4.1$ & $2.4$ & $1.37$\\
\emph{Address} & $4.6$ & $5$ &  & $4.6$ & $5$ &  & $6.6$ & $6.3$ &  & $8.7$ & $7.4$ & $4.40$\\
\emph{Jmlr} & $4.7$ & $5$ &  & $4.7$ & $5$ &  & $6.6$ & $6.2$ &  & $8.5$ & $6.8$ & $3.55$\\
\emph{Nsf} & $7.3$ & $9.7$ &  & $7.3$ & $9.7$ &  & $14.3$ & $18.1$ &  & $39.6$ & $49.3$ & $1.09$\\
\bottomrule
\end{tabular}
\end{center}
\label{tab:sizes}
\end{table}

From these results we see that the number of nodes and edges stay small. This
is due to the fact that majority of episodes are small, typically with 2--3
nodes. Simplification does not add any new nodes or edges since we use strict
episodes, where nodes with the same label must be connected, consequently,
$\mach{\efam{G}}$ is simple. Number of nodes and edges are at highest for
$M^*$, a machine needed to compute cross-moments for \emph{Nsf} data. This is
due to the fact that \emph{Nsf} contains a lot of phrases where the same words
are being repeated. As a consequence, we discover large episodes which in turn
generate large machines. Running times given in the last column of
Table~\ref{tab:sizes} imply that ranking is fast. Ranking discovered episodes
is done within few seconds.  For example, in \emph{Address} ranking $40\,000$
episodes takes less than 5 seconds.

We consider only closed and strict episodes as candidates. If we consider also
non-closed episodes, then the distribution of episode types may change as long
closed episodes tend to be serial. Consequently, we will have more general
episodes. This may result in larger machines as serial episodes have the
simplest machines.

\subsection{Significant Episodes}

Let us first consider \emph{Plant} dataset. The first 5 episodes according to
our ranking were exactly the planted patterns. The scores of these patterns are
between $99\,500$ and $84\,000$. The following patterns are typically a
combination of an original pattern with an additional parallel symbol or a
subset of an original pattern. The scores of these patterns, though significant,
are dropping fast: the score of the 6th pattern is $67\,000$, the score of 7th
pattern is $42\,000$.
Note that if we used frequency (or any other monotonic measure) as a score, subsets
of these planted patterns would have appeared first in the list.

Our next step is to see what types of episodes does our score preferred. In
order to do that, we first consider Figure~\ref{fig:nodecnt} where we have
plotted the number of nodes in an episode as a function of rank.
We see that top patterns tend to have more nodes. This is especially prominent
with \emph{Address} and \emph{Nsf} datasets.

\begin{figure}[htb!]
\begin{center}
\begin{tikzpicture}
\begin{axis}[xlabel={rank}, ylabel= {number of nodes},
    width = 10cm,
    height = 5cm,
    xmax = 1000,
    ymin = 1.8,
    ymax = 5.5,
    cycle list name=yaf,
    no markers,
    xticklabel = {$\scriptstyle\pgfmathprintnumber{\tick}$},
    yticklabel = {$\scriptstyle\pgfmathprintnumber{\tick}$},
	legend pos = outer north east
    ]

\addplot table[x expr = {\lineno * 10 + 1}, y index = 0, header = false] {mobytopnodecnt.dat};
\addplot table[x expr = {\lineno * 10 + 1}, y index = 0, header = false] {addresstopnodecnt.dat};
\addplot table[x expr = {\lineno * 10 + 1}, y index = 0, header = false] {jmlrtopnodecnt.dat};
\addplot table[x expr = {\lineno * 10 + 1}, y index = 0, header = false] {nsftopnodecnt.dat};

\legend{\emph{Moby}, \emph{Address}, \emph{Jmlr}, \emph{Nsf}};
\pgfplotsextra{\yafdrawaxis{1}{1000}{1.8}{5.5}}
\end{axis}
\end{tikzpicture}
\end{center}
\caption{Number of nodes in top-$1\,000$ episodes as a function of rank. Counts are smoothed by computing averages of batches of ten episodes}
\label{fig:nodecnt}
\end{figure}
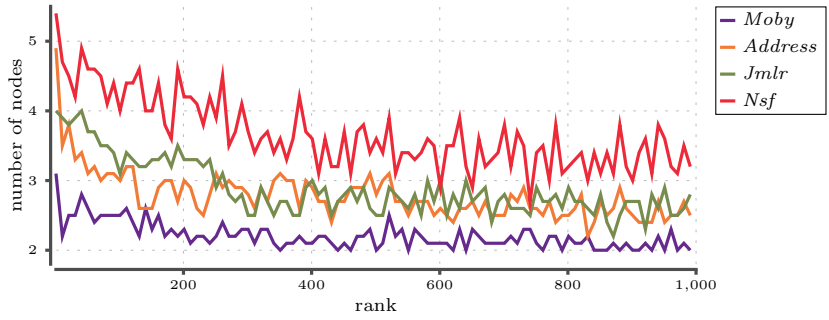

We continued our experiments by computing the proportion of episode types, that
is, whether an episode is a parallel, serial, or general, as a function of
rank, given in Figure~\ref{fig:type}.  From figures we see that distribution
depends heavily on a sequence. Serial episodes tend to be distributed evenly,
parallel episodes tend to be missing from the very top and general episodes
tend to be missing from the very bottom.

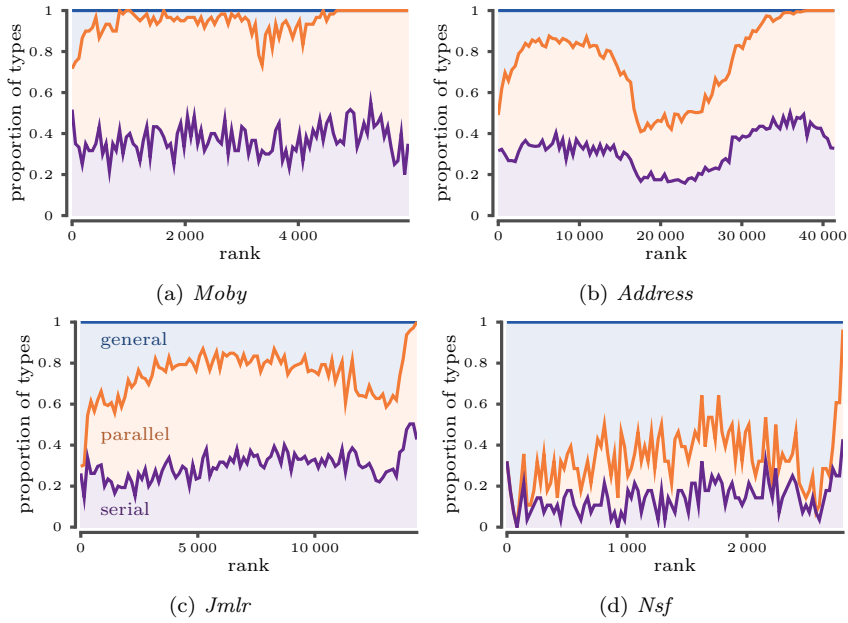
\begin{figure}[htb!]
\begin{center}
\subcaptionbox{\emph{Moby}}{
\begin{tikzpicture}
\begin{axis}[xlabel={rank}, ylabel= {proportion of types},
    width = 6cm,
    height = 4.3cm,
    xmax = 5940,
    ymin = 0,
    ymax = 1,
    cycle list name=yaf,
    no markers,
	scaled x ticks = false,
	x tick label style = {/pgf/number format/set thousands separator = {\,}},
    xticklabel = {$\scriptstyle\pgfmathprintnumber{\tick}$},
    yticklabel = {$\scriptstyle\pgfmathprintnumber{\tick}$},
	grid=none,
    ]

\begin{scope}
\path[clip] (axis cs: 0, 0) rectangle (axis cs: 5940, 1.005);
\addplot[fill, yafcolor5!10] coordinates { (0, 1) (5940, 1) } \closedcycle;
\addplot+[fill, yafcolor2!10] table[x index = 0, y index = 2, header = false] {mobytoptype.dat} \closedcycle;
\addplot+[fill, yafcolor1!10] table[x index = 0, y index = 1, header = false] {mobytoptype.dat} \closedcycle;
\addplot[fill, yafcolor5] coordinates { (0, 1) (5940, 1) };
\addplot+[yafcolor2] table[x index = 0, y index = 2, header = false] {mobytoptype.dat};
\addplot+[yafcolor1] table[x index = 0, y index = 1, header = false] {mobytoptype.dat};
\end{scope}

\pgfplotsextra{\yafdrawaxis{1}{5940}{0}{1}}
\end{axis}
\end{tikzpicture}}
\subcaptionbox{\emph{Address}}{
\begin{tikzpicture}
\begin{axis}[xlabel={rank}, ylabel= {proportion of types},
    width = 6cm,
    height = 4.3cm,
    xmax = 41382,
    ymin = 0,
    ymax = 1,
    cycle list name=yaf,
    no markers,
	scaled x ticks = false,
	x tick label style = {/pgf/number format/set thousands separator = {\,}},
    xticklabel = {$\scriptstyle\pgfmathprintnumber{\tick}$},
    yticklabel = {$\scriptstyle\pgfmathprintnumber{\tick}$},
	grid=none,
    ]

\begin{scope}
\path[clip] (axis cs: 0, 0) rectangle (axis cs: 41382, 1.005);
\addplot[fill, yafcolor5!10] coordinates { (0, 1) (41382, 1) } \closedcycle;
\addplot+[fill, yafcolor2!10] table[x index = 0, y index = 2, header = false] {addresstoptype.dat} \closedcycle;
\addplot+[fill, yafcolor1!10] table[x index = 0, y index = 1, header = false] {addresstoptype.dat} \closedcycle;
\addplot[fill, yafcolor5] coordinates { (0, 1) (41382, 1) };
\addplot+[yafcolor2] table[x index = 0, y index = 2, header = false] {addresstoptype.dat};
\addplot+[yafcolor1] table[x index = 0, y index = 1, header = false] {addresstoptype.dat};
\end{scope}

\pgfplotsextra{\yafdrawaxis{1}{41382}{0}{1}}
\end{axis}
\end{tikzpicture}}
\subcaptionbox{\emph{Jmlr}}{
\begin{tikzpicture}
\begin{axis}[xlabel={rank}, ylabel= {proportion of types},
    width = 6cm,
    height = 4.3cm,
    xmax = 14355,
    ymin = 0,
    ymax = 1,
    cycle list name=yaf,
    no markers,
	scaled x ticks = false,
	x tick label style = {/pgf/number format/set thousands separator = {\,}},
    xticklabel = {$\scriptstyle\pgfmathprintnumber{\tick}$},
    yticklabel = {$\scriptstyle\pgfmathprintnumber{\tick}$},
	grid=none,
    ]

\begin{scope}
\path[clip] (axis cs: 0, 0) rectangle (axis cs: 14355, 1.005);
\addplot[fill, yafcolor5!10] coordinates { (0, 1) (14355, 1) } \closedcycle;
\addplot+[fill, yafcolor2!10] table[x index = 0, y index = 2, header = false] {jmlrtoptype.dat} \closedcycle;
\addplot+[fill, yafcolor1!10] table[x index = 0, y index = 1, header = false] {jmlrtoptype.dat} \closedcycle;
\addplot[fill, yafcolor5] coordinates { (0, 1) (14355, 1) };
\addplot+[yafcolor2] table[x index = 0, y index = 2, header = false] {jmlrtoptype.dat};
\addplot+[yafcolor1] table[x index = 0, y index = 1, header = false] {jmlrtoptype.dat};
\node[anchor = west, font = \scriptsize, yafcolor5!70!black] at (axis cs:500, 0.91) {general};
\node[anchor = west, font = \scriptsize, yafcolor2!70!black] at (axis cs:500, 0.45) {parallel};
\node[anchor = west, font = \scriptsize, yafcolor1!70!black] at (axis cs:500, 0.09) {serial};
\end{scope}

\pgfplotsextra{\yafdrawaxis{1}{14355}{0}{1}}
\end{axis}
\end{tikzpicture}}
\subcaptionbox{\emph{Nsf}}{
\begin{tikzpicture}
\begin{axis}[xlabel={rank}, ylabel= {proportion of types},
    width = 6cm,
    height = 4.3cm,
    xmax = 2800,
    ymin = 0,
    ymax = 1,
    cycle list name=yaf,
    no markers,
	scaled x ticks = false,
	x tick label style = {/pgf/number format/set thousands separator = {\,}},
    xticklabel = {$\scriptstyle\pgfmathprintnumber{\tick}$},
    yticklabel = {$\scriptstyle\pgfmathprintnumber{\tick}$},
	grid=none,
    ]

\begin{scope}
\path[clip] (axis cs: 0, 0) rectangle (axis cs: 2800, 1.005);

\addplot[fill, yafcolor5!10] coordinates { (0, 1) (2800, 1) } \closedcycle;
\addplot+[fill, yafcolor2!10] table[x index = 0, y index = 2, header = false] {nsftoptype.dat} \closedcycle;
\addplot+[fill, yafcolor1!10] table[x index = 0, y index = 1, header = false] {nsftoptype.dat} \closedcycle;
\addplot[fill, yafcolor5] coordinates { (0, 1) (2800, 1) };
\addplot+[yafcolor2] table[x index = 0, y index = 2, header = false] {nsftoptype.dat};
\addplot+[yafcolor1] table[x index = 0, y index = 1, header = false] {nsftoptype.dat};
\end{scope}

\pgfplotsextra{\yafdrawaxis{1}{2800}{0}{1}}
\end{axis}
\end{tikzpicture}}

\end{center}
\caption{Proportions of different types of episodes as a function of rank. The
top area corresponds go the general episodes, the middle area represents
parallel episodes and the bottom area represents serial episodes.
Proportions were computed by dividing the ranked patterns into 100 bins}
\label{fig:type}
\end{figure}

Finally, let us conclude by demonstrating some of the discovered top patterns
from \emph{Address} and \emph{Jmlr} datasets, given in Figure~\ref{fig:found}.
The first three patterns represent phrases that are often said by the
presidents. Episode in Figure~\ref{fig:found:b} is particularly interesting
since presidents tend to acknowledge vice president(s) and the chief justice at
the beginning of their speeches but the order is not fixed. The remaining 3 patterns
represent common phrases occurring in abstracts of machine learning articles.

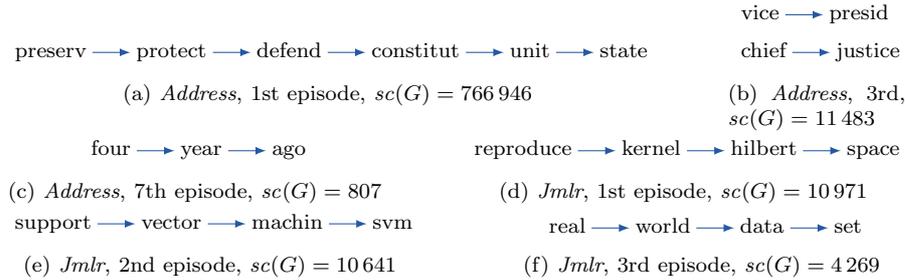
\begin{figure}[htb!]
\subcaptionbox{\emph{Address}, 1st episode, $\score{G} = 766\,946$}{
\begin{tikzpicture}[node distance = 5mm]
\coordinate (n0);
\nodepath{preserv, protect, defend, constitut, unit, state}{n}
\edgepath{n1/n2,n2/n3,n3/n4,n4/n5,n5/n6}
\end{tikzpicture}}\hfill
\subcaptionbox{\emph{Address}, 3rd, $\score{G} = 11\,483$\label{fig:found:b}}{
\begin{tikzpicture}[node distance = 5mm]
\coordinate (n0);
\nodepath{vice, presid}{n}
\edgepath{n1/n2}
\coordinate[below=of n0] (a0);
\nodepath{chief, justice}{a}
\edgepath{a1/a2}
\end{tikzpicture}}

\subcaptionbox{\emph{Address}, 7th episode, $\score{G} = 807$}{
\hspace*{1cm}\begin{tikzpicture}[node distance = 5mm]
\coordinate (n0);
\nodepath{four, year, ago}{n}
\edgepath{n1/n2, n2/n3}
\end{tikzpicture}\hspace*{1cm}}\hfill
\subcaptionbox{\emph{Jmlr}, 1st episode, $\score{G} = 10\,971$}{
\begin{tikzpicture}[node distance = 5mm]
\coordinate (n0);
\nodepath{reproduce, kernel, hilbert, space}{n}
\edgepath{n1/n2, n2/n3, n3/n4}
\end{tikzpicture}}

\subcaptionbox{\emph{Jmlr}, 2nd episode, $\score{G} = 10\,641$}{
\begin{tikzpicture}[node distance = 5mm]
\coordinate (n0);
\nodepath{support, vector, machin, svm}{n}
\edgepath{n1/n2, n2/n3, n3/n4}
\end{tikzpicture}}\hfill
\subcaptionbox{\emph{Jmlr}, 3rd episode, $\score{G} = 4\,269$}{
\hspace*{0.5cm}\begin{tikzpicture}[node distance = 5mm]
\coordinate (n0);
\nodepath{real, world, data, set}{n}
\edgepath{n1/n2, n2/n3, n3/n4}
\end{tikzpicture}\hspace*{0.5cm}}
\caption{Examples of highly ranked episodes from \emph{Address} and \emph{Jmlr} datasets}
\label{fig:found}
\end{figure}

\subsection{Asymptotic normality}
\label{sec:normality}

Proposition~\ref{prop:rationormal} implies that if the independence assumption hold
in the testing sequence, then $\score{G}$ should behave like a sample from a
standard normal distribution as the size of the sequence increases.  In this
section we test the rate of convergence.

To that end we generated several sequences with independent events, each event
having equal probability to occur. We generated three training sequences from
alphabets of $100$, $500$, and $1\,000$ symbols. Each sequence contained
$10\,000$ events.  For each training sequence we generated $3$ testing
sequences of different lengths, namely $10^4$, $10^5$, and $10^6$.

From each testing sequence we mined frequent episodes. We selected the
thresholds such that we got roughly $10\,000$ episodes, more specifically, we
used $12$, $3$, $2$ as thresholds for sequences with $100$, $500$, $1\,000$
symbols respectively. We then tested the discovered non-singleton episodes on
testing sequences. Note that computing the score requires probabilities of
individual events. We computed the scores both by using the true probabilities
and by estimating the probabilities from the training sequence.

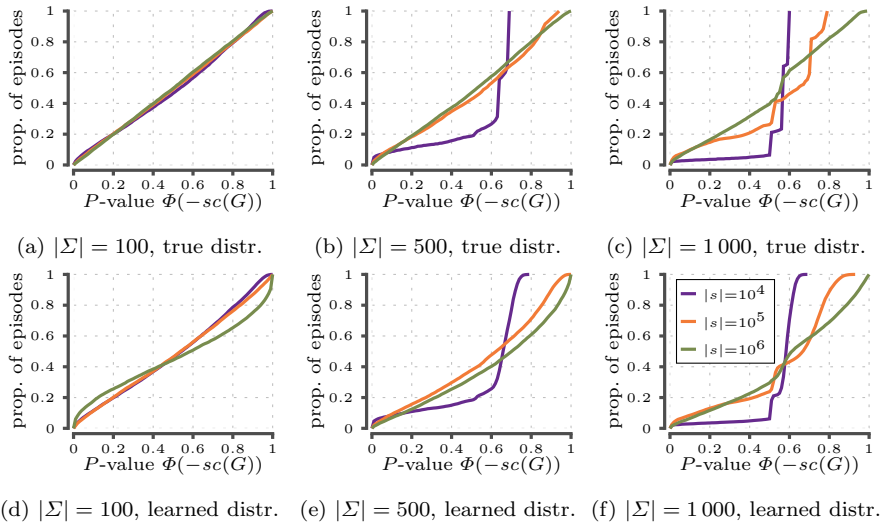
\begin{figure}

\begin{center}
\subcaptionbox{\label{fig:asymptotic:a}$\abs{\Sigma} = 100$, true distr.}{
\begin{tikzpicture}
\begin{axis}[xlabel={$P$-value $\Phi(-\score{G})$}, ylabel= {prop. of episodes},
    width = 4.2cm,
    xmax = 1,
    ymin = 0,
    ymax = 1,
    cycle list name=yaf,
    no markers,
	legend pos = north west,
    xticklabel = {$\scriptstyle\pgfmathprintnumber{\tick}$},
    yticklabel = {$\scriptstyle\pgfmathprintnumber{\tick}$}
    ]
\addplot table[x index = 1, y index = 0, header = false] {res_100_e4.dat};
\addplot table[x index = 1, y index = 0, header = false] {res_100_e5.dat};
\addplot table[x index = 1, y index = 0, header = false] {res_100_e6.dat};
\pgfplotsextra{\yafdrawaxis{0}{1}{0}{1}}
\end{axis}
\end{tikzpicture}}
\subcaptionbox{$\abs{\Sigma} = 500$, true distr.}{
\begin{tikzpicture}
\begin{axis}[xlabel={$P$-value $\Phi(-\score{G})$}, ylabel= {prop. of episodes},
    width = 4.2cm,
    xmax = 1,
    ymin = 0,
    ymax = 1,
    cycle list name=yaf,
    no markers,
    xticklabel = {$\scriptstyle\pgfmathprintnumber{\tick}$},
    yticklabel = {$\scriptstyle\pgfmathprintnumber{\tick}$}
    ]
\addplot table[x index = 1, y index = 0, header = false] {res_500_e4.dat};
\addplot table[x index = 1, y index = 0, header = false] {res_500_e5.dat};
\addplot table[x index = 1, y index = 0, header = false] {res_500_e6.dat};
\pgfplotsextra{\yafdrawaxis{0}{1}{0}{1}}
\end{axis}
\end{tikzpicture}}
\subcaptionbox{$\abs{\Sigma} = 1\,000$, true distr.}{
\begin{tikzpicture}
\begin{axis}[xlabel={$P$-value $\Phi(-\score{G})$}, ylabel= {prop. of episodes},
    width = 4.2cm,
    xmax = 1,
    ymin = 0,
    ymax = 1,
    cycle list name=yaf,
    no markers,
    xticklabel = {$\scriptstyle\pgfmathprintnumber{\tick}$},
    yticklabel = {$\scriptstyle\pgfmathprintnumber{\tick}$}
    ]
\addplot table[x index = 1, y index = 0, header = false] {res_1000_e4.dat};
\addplot table[x index = 1, y index = 0, header = false] {res_1000_e5.dat};
\addplot table[x index = 1, y index = 0, header = false] {res_1000_e6.dat};
\pgfplotsextra{\yafdrawaxis{0}{1}{0}{1}}
\end{axis}
\end{tikzpicture}}

\subcaptionbox{\label{fig:asymptotic:d}$\abs{\Sigma} = 100$, learned distr.}{
\begin{tikzpicture}
\begin{axis}[xlabel={$P$-value $\Phi(-\score{G})$}, ylabel= {prop. of episodes},
    width = 4.2cm,
    xmax = 1,
    ymin = 0,
    ymax = 1,
    cycle list name=yaf,
    no markers,
    xticklabel = {$\scriptstyle\pgfmathprintnumber{\tick}$},
    yticklabel = {$\scriptstyle\pgfmathprintnumber{\tick}$}
    ]
\addplot table[x index = 1, y index = 0, header = false] {res_100_e4_learned.dat};
\addplot table[x index = 1, y index = 0, header = false] {res_100_e5_learned.dat};
\addplot table[x index = 1, y index = 0, header = false] {res_100_e6_learned.dat};
\pgfplotsextra{\yafdrawaxis{0}{1}{0}{1}}
\end{axis}
\end{tikzpicture}}
\subcaptionbox{$\abs{\Sigma} = 500$, learned distr.}{
\begin{tikzpicture}
\begin{axis}[xlabel={$P$-value $\Phi(-\score{G})$}, ylabel= {prop. of episodes},
    width = 4.2cm,
    xmax = 1,
    ymin = 0,
    ymax = 1,
    cycle list name=yaf,
    no markers,
    xticklabel = {$\scriptstyle\pgfmathprintnumber{\tick}$},
    yticklabel = {$\scriptstyle\pgfmathprintnumber{\tick}$},
    ]
\addplot table[x index = 1, y index = 0, header = false] {res_500_e4_learned.dat};
\addplot table[x index = 1, y index = 0, header = false] {res_500_e5_learned.dat};
\addplot table[x index = 1, y index = 0, header = false] {res_500_e6_learned.dat};
\pgfplotsextra{\yafdrawaxis{0}{1}{0}{1}}
\end{axis}
\end{tikzpicture}}
\subcaptionbox{$\abs{\Sigma} = 1\,000$, learned distr.}{
\begin{tikzpicture}
\begin{axis}[xlabel={$P$-value $\Phi(-\score{G})$}, ylabel= {prop. of episodes},
    width = 4.2cm,
    xmax = 1,
    ymin = 0,
    ymax = 1,
    cycle list name=yaf,
    no markers,
	legend pos = north west,
    xticklabel = {$\scriptstyle\pgfmathprintnumber{\tick}$},
    yticklabel = {$\scriptstyle\pgfmathprintnumber{\tick}$}
    ]
\addplot table[x index = 1, y index = 0, header = false] {res_1000_e4_learned.dat};
\addplot table[x index = 1, y index = 0, header = false] {res_1000_e5_learned.dat};
\addplot table[x index = 1, y index = 0, header = false] {res_1000_e6_learned.dat};
\legend{$\scriptstyle{\abs{s} = 10^4}$, $\scriptstyle{\abs{s} = 10^5}$, $\scriptstyle{\abs{s} = 10^6}$};
\pgfplotsextra{\yafdrawaxis{0}{1}{0}{1}}
\end{axis}
\end{tikzpicture}}
\end{center}

\caption{Cumulative proportion of episodes as a function of a score
$\Phi\pr{-\score{G}}$ in generated sequences with independent events.  Ideally,
the proportion is an identity function. The left column represents sequences
with $100$ symbols, the centre column represents sequences with $500$ symbols,
and the right column represents sequences with $1\,000$ symbols.  The top row
uses true occurrences for individual symbols when computing the moments, while
the bottom row estimates the occurrences from training sequence. Each plot
contains three lines representing different sizes of testing sequences}
\label{fig:asymptotic}
\end{figure}

In Figure~\ref{fig:asymptotic} we plotted the proportion of episodes for which
$\Phi\pr{-\score{G}}$ is smaller than the threshold.
Proposition~\ref{prop:rationormal} implies that ideally this plot should be the
identity line between $0$ and $1$. We see that this is the case in
Figure~\ref{fig:asymptotic:a}. As we increase the size of the alphabet, the
estimate becomes more and more inaccurate. We believe that this is due to high
skewness of the actual distribution. When using true probabilities for
individual probabilities, longer testing sequences produce better results.
Using estimated values introduces additional errors, as can be seen in
Figure~\ref{fig:asymptotic:d} where a testing sequence of length $10^6$ is
less ideal than the sequence of $10^5$. However, this phenomenon can be
attacked by dividing the sequence to training and testing portion more fairly,
thus making the estimates more accurate. 

\section{Discussion and Conclusions}
\label{sec:conclusions}

In this paper we proposed a new quality measure for episodes based on
minimal windows.  In order to do this, we approached by computing the expected
values based on the independence model and compared the expectations to  
the observed values by computing a $Z$-score.

Our main technical contribution is a technique for computing the moments of
minimal windows. In order to do so we created a series of elaborate finite
state machines and demonstrated that we can compute the moments recursively. In
this paper we chose to use a specific statistic, namely $\rho^d$, where $d$ is
a length of a minimal window and $\rho$ is a user-given parameter.  However,
the same principle can be applied also directly on the length of minimal 
windows.

While the actual computation of statistics is fairly complex and requires a
great number of recursive updates, and even may be exponentially slow, our
experiments demonstrate that the computation is fast in practice, we can rank
tens of thousands of episodes in the matter of seconds.

Our technique has its limitations. In synthetic data, \emph{plant}, after
finding 5 true patterns, our method continued scoring high patterns that were
either superpatterns of subpatterns of the first 5 patterns. All these patterns
are significant in the sense that they deviate significantly from the
independence model. Nevertheless, they provide no new information about the
underlying structure in the data. This problem occurs in any pattern ranking
scheme where the ranking method does not take other patterns into account.

Approaches to further reduce patterns by considering patterns as a set instead
of individual patterns have been developed for itemsets. For example, one
approach for itemsets involve in partitioning itemsets into subitemsets and
applying independence assumption between the individual
parts~\cite{webb:10:self-sufficient}. Transforming this idea to episodes is not
trivial. A more direct approach---although using only serial
episodes---where episodes were selected using MDL techniques
was suggested in~\cite{tatti:12:long}. An extension of this work to
general episodes would be interesting.

Proposition~\ref{prop:rationormal} implies that we can interpret our measure as a
$P$-value. In practice, this can be problematic as we demonstrate in
Section~\ref{sec:normality}. Since the distributions are heavily skewed,
especially when dealing with a large alphabet, we require a lot of samples
before the normality assumption becomes accurate. Nevertheless our experiments
with synthetic and text data demonstrate that our score produces interpretable
rankings.

\section*{Acknowledgements}
Nikolaj Tatti was partly supported by a Post-Doctoral Fellowship of the Research Foundation -- Flanders (\textsc{fwo}).

\bibliographystyle{spbasic}
\bibliography{bibliography}

\appendix
\section{Proofs}
\label{sec:appendix}

\begin{proof}[Proof of Proposition~\ref{prop:simplecover}]
We will prove this by induction. Let $i$ be the source state of $M$.  The
proposition holds trivially when $X = \set{i}$, a source state. Assume now that
the proposition holds for all parent states of $X$.

Assume that $s$ covers $X$.  Let $t$ be a subsequence of $s$ that leads
$\simple{M}$ from the source state $\set{i}$ to $X$.  Let $s_e$ be the last
symbol of $s$ occurring in $t$.  Then a parent state $Y = \enset{y_1}{y_L} = \parent{X; s_e}$
is covered by $s[1, e - 1]$.  By the induction assumption at least one $y_k$ is
covered by $s[1, e - 1]$. If there is $x_j \in X$ such that $x_j = y_k$, then $x_j$ is covered by
$s$, otherwise there is $x_j$ that has $y_k$ as a parent state.  The edge
connecting $x_j$ and $y_k$ is labelled with $s_e$. Hence $s$ covers $x_j$ also.

To prove the other direction assume that $s$ covers $x_j$. Let $t$ be a
sub-sequence that leads $M$ from $i$ to $x_j$. Let $s_e$ be the last symbol
occurring in $t$. Let $y$ be the parent state of $x_j$ connected by an edge
labelled with $s_e$. Since $s_e \in \inc{X}$, we must have $Y$ as a parent
state of $X$ such that $y \in Y$. By the induction assumption, $s[1, e - 1]$
covers $Y$.  Hence $s$ covers $X$.
\end{proof}

In order to prove Proposition~\ref{prop:greedy} we need the following lemma.

\begin{lemma}
\label{lem:greedy}
Let $G$ be an episode and assume a sequence $s = \enpr{s_1}{s_L}$ that covers $G$.
Let $\efam{H} = \set{G - v; v \in \sinks{G}, \lab{v} = s_L}$. If $\efam{H}$ is
empty, then $s[1, L - 1]$ covers $G$. Otherwise,
there is an episode $H \in \efam{H}$ that is covered by $s[1, L - 1]$.
\end{lemma}

\begin{proof}
Let $f$ be a valid mapping of $V(G)$ to indices of $s$ corresponding to the
coverage. If $\efam{H}$ is empty, then $L$ is not in the range of $f$, then $s[1, L -
1]$ covers $G$. If $\efam{H}$ is not empty but $L$ is not in the range of $f$, then $s[1,
L - 1]$ covers $G$, and any episode in $\efam{H}$. 

Assume now that $L$ is in range of $f$, that is, there is a sink $v$ with
a label $s_L$. Episode $G - v$ is in $\efam{H}$. Moreover, $f$ restricted to $G - v$
provides the needed mapping in order to $s[1, L - 1]$ to cover $G - v$.
\end{proof}

\begin{proof}[Proof of Proposition~\ref{prop:greedy}]
If $\greedy{X, s} = \set{i}$, then it is trivial to see that $s$ covers $X$.

Assume that $s$ covers $X$.  We will prove this direction by induction over
$L$, the length of $s$.  The proposition holds for $L = 0$. Assume that $L > 0$
and that proposition holds for all sequences of length $L - 1$. 

Let $Y = \greedy{X, s_L}$. Note that $\greedy{X, s} = \greedy{Y, s[1, L - 1]}$.
Hence, to prove the proposition we need to show that $s[1, L - 1]$ covers $Y$.

If $Y = \set{i}$, then $s[1, L - 1]$ covers $Y$.  Hence, we can assume that $Y
\neq \set{i}$, that is, $Y = \sub{X; s_L} \cup \stay{X; s_L}$.

Proposition~\ref{prop:simplecover} implies that one of the states of $M_G$, say
$x \in X$, is covered by $s$.  Proposition~\ref{prop:cover} states that the
corresponding episode, say $H$, is covered by $s$.

Assume that $x \in Y$. This is possibly only if $x \in \stay{X; s_L}$ that is
there is no sink node in $H$ labelled as $s_L$. Lemma~\ref{lem:greedy} implies
that $s[1, L - 1]$ covers $H$, Propositions~\ref{prop:cover}~and~\ref{prop:simplecover}
imply that $s[1, L - 1]$ covers $Y$.

Assume that $x \notin Y$, Then $\sub{X; s_L} \subseteq Y$ contains all states
of $M_G$ corresponding to the episodes of form $H - v$, where $v$ is sink node
of $H$ with a label $s_L$.  According to Lemma~\ref{lem:greedy}, $s[1, L - 1]$
covers one of these episodes,
Propositions~\ref{prop:cover}~and~\ref{prop:simplecover} imply that $s[1, L -
1]$ covers $Y$.
\end{proof}

\begin{proof}[Proof of Proposition~\ref{prop:join}]
We will prove the proposition by induction over $L$, the length of $s$.  The
proposition holds when $L = 0$. Assume that $L > 0$ and that proposition holds
for sequence of length $L - 1$.

Let $\beta = (y_1, y_2) = \greedy{\alpha, s_L}$. Then, by definition of $M^*$,
$y_i = \greedy{x_i, s_L}$. Write $t = s[1, L - 1]$.
Since 
\[
	\greedy{\beta, t} = \greedy{\alpha, s}, \quad
	\greedy{y_1, t} = \greedy{x_1, s}, \quad
	\greedy{y_2, t} = \greedy{x_2, s}.
\]
and, because of induction assumption, $\greedy{\beta, t} = (\greedy{y_1, t}, \greedy{y_2, t})$,
we have $\greedy{\alpha, s} = (\greedy{x_1, s}, \greedy{x_2, s})$.
\end{proof}

\begin{proof}[Proof of Proposition~\ref{prop:minmach}]
Assume that $s$ is a minimal window for $G$. Since $s$ covers $S$ in $M$,
$\greedy{S, s; M} = I$. This implies that $\greedy{S, s; M_1} = I$ or
$\greedy{S, s; M_1} = J$. The latter case implies that $s[2, L]$ covers $S$ in
$M$, which is a contradiction. Hence, $\greedy{S, s; M_1} = I$.  Let $Z =
\greedy{T, s; M_2}$.  If $Z = I$, then $s[1, L - 1]$ covers $S$ in $M$, which
is a contradiction.  Hence $Z \neq I$. Proposition~\ref{prop:join} implies that
$\greedy{\alpha, s} = (I, Z)$.

Assume that $\greedy{\alpha, s} = (I, Y)$ such that $Y \neq I$.
Proposition~\ref{prop:join} implies that $\greedy{S, s; M_1} = I$ and
$\greedy{T, s; M_2} \neq I$. The former implication leads to $\greedy{S, s; M}
= I$ which implies that $s$ covers $G$.

If $s[2, L]$ covers $G$, then $\greedy{S, s[2, L]; M} = I$ and so $\greedy{S,
s; M_1} = J$, which is a contradiction. Hence $s[2, L]$ does not cover $G$.
The latter implication leads to $\greedy{S, s[1, L - 1]; M} \neq I$ which
implies that $s[1, L - 1]$ does not cover $G$. This proves the proposition.
\end{proof}

\begin{proof}[Proof of Proposition~\ref{prop:greedyprob}]
If $L = 0$, then $\greedy{x, s} = x$ which immediately implies the proposition.
Assume that $L > 0$. Note that $\greedy{x, s} = \greedy{\greedy{x, s_L}, s[1, L - 1]}$.
\[
\begin{split}
	&p(\greedy{x, s} \in Y \mid \abs{s} = L) \\
	&\quad = \sum_{a \in \Sigma} p(a) p(\greedy{x, s} \in Y \mid \abs{s} = L, s_L = a)\\
	&\quad = \sum_{a \in \Sigma} p(a) p(\greedy{\greedy{x, a}, s[1, L - 1]} \in Y \mid \abs{s} = L, s_L = a).\\
\end{split}
\]
Since individual symbols in $s$ are independent, it follows that
\[
	p(\greedy{\greedy{x, a}, s[1, L - 1]} \in Y \mid \abs{s} = L, s_L = a) = \pgreedy{\greedy{x, a},  Y,  L - 1}.
\]
This proves the proposition.
\end{proof}

\begin{proof}[Proof of Lemma~\ref{lem:finite}]
Define $q = \sqrt{1 - \min_{a \in \Sigma} p(a)}$. Note that $q < 1$.
We claim that for each $x$ there is a constant $C_x$
such that $\pgreedy{x, Y, L} \leq C_xq^L = O(q^{L})$ which
in turns proves the lemma.  To prove
the claim we use induction over parenthood of $x$ and $L$.

Since the source node is not in $Y$, the first step follows immediately.
Assume that the result holds for all parent states of $x$.
Define 
\[
C_x = \max\big(1, \frac{1}{q(1 - q)}\sum_{a \in \inc{x} \atop y = \greedy{x, a}} p(a) C_y\big) \text{ which implies }
	q C_x + q^{-1}\sum_{a \in \inc{x} \atop y = \greedy{x, a}} p(a)C_y \leq C_x.
\]
Since $C_x \geq 1$, the case of $L = 0$ holds.
Assume that the the induction assumption holds for $C_y$ and for $C_x$ up to $L - 1$.
Let $r = 1 - \sum_{a \in \inc{x}} p(a)$. Note that $r \leq q^2$.
According to Proposition~\ref{prop:greedyprob} we have
\[
\begin{split}
	\pgreedy{x, Y, L}  & = r \pgreedy{x, Y, L - 1} +  \sum_{a \in \inc{x} \atop y = \greedy{x, a}} p(a) \pgreedy{y,  Y,  L - 1} \\
	& \leq r C_xq^{L - 1} + \sum_{a \in \inc{x} \atop y = \greedy{x, a}} p(a)C_y q^{L - 1}\\
	& \leq q^L\big(q C_x + q^{-1}\sum_{a \in \inc{x} \atop y = \greedy{x, a}} p(a)C_y\big) \leq q^LC_x.
\end{split}
\]
This proves that $\pgreedy{x, Y, L}$ decays at exponential rate.
\end{proof}

\begin{proof}[Proof of Proposition~\ref{prop:greedymom}]
The proposition follows by a straightforward manipulation of Equation~\ref{eq:greedyprob}.
First note that
\begin{equation}
\label{eq:greedymom1}
	\sum_{L = 1}^\infty f(L - 1) \pgreedy{x, Y, L} = c\moment{x, f, Y} + \moment{x, h, Y}.
\end{equation}
Equation~\ref{eq:greedyprob} implies that
\begin{equation}
\label{eq:greedymom2}
\begin{split}
	\sum_{L = 1}^\infty f(L - 1) \pgreedy{x, Y, L}
	                 & = \sum_{a \in \Sigma \atop y = \greedy{x, a}} p(a) \sum_{L = 1}^\infty f(L - 1) \pgreedy{y,  Y,  L - 1} \\
	                 & = \sum_{a \in \Sigma \atop y = \greedy{x, a}} p(a) (i(y) + \sum_{L = 1}^\infty f(L) \pgreedy{y,  Y,  L}) \\
	                 & = \sum_{a \in \Sigma \atop y = \greedy{x, a}} p(a) (i(y) + \moment{y, f, Y}) \\
	                 & = q(i(x) + \moment{x, f, Y})  + \sum_{\mathclap{a \in \inc{x} \atop y = \greedy{x, a}}} p(a) (i(y) + \moment{y, f, Y}). \\
\end{split}
\end{equation}
Combining Equations~\ref{eq:greedymom1}~and~\ref{eq:greedymom2} and solving $\moment{x, f, Y}$ gives us the result.
\end{proof}

To prove the asymptotic normality we will use the following theorem.

\begin{theorem}[Theorem 27.4 in~\citep{billingsley:95:probability}]
\label{thr:mixing}
Assume that $U_k$ is a stationary sequence with $\mean{U_k} = 0$,
$\mean{U_k^{12}} < \infty$, and is $\alpha$-mixing with $\alpha(n) = O(n^{-5})$,
where $\alpha(n)$ is the strong mixing coefficient,
\[
	\alpha(n) = \sup_{k, A, B} \abs{p(A, B) - p(A)p(B)},
\]
where $A$ is an event depending only on $U_{-\infty}, \ldots, U_k$ and $B$ is an event depending
only on $U_{k + n}, \ldots,U_{\infty}$. Let $S_k = U_1 + \cdots + U_k$. 
Then $\sigma^2 = \lim_k 1/k \mean{S_k}$ exists and 
$S_k / \sqrt{k}$ converges to $N(0, \sigma^2)$ and $\sigma^2 = \mean{U_1^2} + 2\sum_{k = 2}^\infty \mean{U_1U_k}$.
\end{theorem}

\begin{proof}[Proof of Proposition~\ref{prop:normal}]
Let us write $T_k = (Z_k, X_k) - (q, p)$ and $S_L = 1/\sqrt{L}\sum_{k = 1}^L
T_k$.  Assume that we are given a vector $r = (r_1, r_2)$ and write $U_k =
r^TT_k$. We will first prove that $r^TS_L$ converges to a normal distribution
using Theorem~\ref{thr:mixing}.

First note that $\mean{U_k} = 0$ and that 
\[
	\mean{U_k^{12}} = \sum_{i = 0}^{12} {12 \choose i}r_1^ir_2^{12 - i}\mean{Z_k^iX_k^{12 - i}} = r_2^{12}\mean{X_k} + \sum_{i = 1}^{12} {12 \choose i}r_1^ir_2^{12 - i}\mean{Z_k^i}.
\]
Since every moment of $Z_k$ and $X_k$ is finite, $\mean{U_k^{12}}$ is also
finite.  We will prove now that $U_k$ is $\alpha$-mixing.

Fix $k$ and $N$. Write $W$ to be an event that $s[k + 1, N]$ covers $G$.  If
$W$ is true, then $X_l$ and $Z_l$ (and hence $U_l$) for $l \leq k$ depends only
$s[l, N]$, that is, either there is a minimal window $s[l, N']$, where $N' < N$
or $X_l = Z_l = 0$.

Let  $A$ be an event depending only on $U_{-\infty}, \ldots, U_k$ and $B$ be an
event depending only on $U_{N + 1}, \ldots,U_{\infty}$.  Then $p(A,B \mid W) =
p(A \mid W)p(B \mid W)$. We can rephrase this and bound $\alpha(n) \leq p(s[1,
n - 1] \text{ does not covers } G)$.  To bound the right side, let $M =
\simple{M_G}$, let $v$ be its sink state and let $V$ be all states save the
source state. 
Then the probability is equal to 
\[
	p(s[1, n - 1] \text{ does not covers } G) = \pgreedy{v, V, n - 1}.
\]
Since $V$ does not contain the source node, the moment $\moment{v, n \to n^5, V}$ is finite.
Consequently, $n^5\pgreedy{v, V, n} \to 0$ which implies that $\alpha(n) = O(n^{-5})$.
Thus Theorem~\ref{thr:mixing} implies that $r^TS_L$ converges to a normal distribution with
the variance $\sigma^2 = r_1^2C_{11} + 2r_1r_2C_{12} + r_2^2C_{22} = r^TCr$.
Levy's continuity theorem~\citep[Theorem 2.13][]{vaart:98:asymptotic} now implies
that the characteristic function of $r^TS_L$ converges to a characteristic function of normal distribution $N(0, \sigma^2)$,
\[
	\mean{\exp\fpr{itr^TS_L}}  \to \exp\fpr{-1/2t^2r^TCr}.
\]
The left side is a characteristic function of $S_L$ (with $tr$ as a parameter).
Similarly, the right side is a characteristic function of $N(0, C)$.
Levy's continuity theorem now implies that $S_L$ converges into $N(0, C)$.
\end{proof}

\begin{proof}[Proof of Proposition~\ref{prop:rationormal}]
Function $f(x, y) = x/y$ is differentiable at $(q, p)$.
Since $1/\sqrt{L}\big(\sum_{k = 1}^L (Z_k, X_k) - (q, p)\big)$ converges to normal distribution, we can apply Theorem~3.1 in~\citep{vaart:98:asymptotic}
so that
\[
	    \sqrt{L}\pr{\frac{\sum_{k = 1}^L Z_k}{\sum_{k = 1}^L X_k} - \mu} = \sqrt{L} f\bigg(1/L\sum_{k = 1}^L Z_k, 1/L\sum_{k = 1}^L X_k\bigg) - \sqrt{L}f(q, p)
\]
converges to $N(0, \sigma^2)$, where $\sigma^2 = \nabla f(q, p)^T C \nabla f(q, p)$. The gradient of $f$ is equal to
$\nabla f(q, p) = (1/p, -\mu/p)$. The proposition follows.
\end{proof}

\begin{proof}[Proof of Proposition~\ref{prop:cross}]
To prove all four cases simultaneously, let us write
write $A$ to be either $X_1$ or $Z_1$ and let $B_k$ to be either $X_k$ or $Z_k$.
Let $a = \mean{A}$ and $b = \mean{B_k}$. First note that
$\mean{(A - a)(B_k - b)} = \mean{A(B_k - b)}$, which allows us to ignore $a$
inside the mean.

Assume that we have $0 < n < k$. Then given that $Y_1 = n$, $A$ and $X_1$ depends only on
$n$ first symbols of sequence. Since $B_k$ does not depend on $k - 1$ first symbols, this implies that
\[
\begin{split}
	p(A, B_k \mid Y_1 = n) = p(A \mid Y_1 = n)p(B_k \mid Y_1  = n) = p(A \mid Y_1 = n)p(B_k),
\end{split}
\]
which in turns implies that $\mean{A (B_k - b) \mid Y_1 = n}  = 0$.

Note that for $A = 0$ whenever $Y_1 = 0$. Consequently, we
have 
\[
\begin{split}
	\means{A\sum_{k = 2}^\infty (B_k - b)}  & = \sum_{n = 1}^\infty \means{A\sum_{k = 2}^\infty (B_k - b) \mid Y_1 = n} p(Y_1 = n) \\
	                                          & = \sum_{n = 1}^\infty \means{A\sum_{k = 2}^n (B_k - b) \mid Y_1 = n} p(Y_1 = n) \\
	                                          & = \means{A\sum_{k = 2}^{Y_1} (B_k - b)} = \means{A \sum_{k = 2}^{Y_1} B_k} - \means{A \sum_{k = 2}^{Y_1}b} \\
	                                          & = \means{A \sum_{k = 2}^{Y_1} B_k} - \mean{A(Y_1 - X_1)}b \\
	                                          & = \means{X_1A \sum_{k = 2}^{Y_1} X_kB_k} - \mean{A(Y_1 - X_1)}b,
\end{split}
\]
where the second last equality holds because $\sum_{k = 2}^{Y_1} 1 = Y_1 - X_1$
and the last equality follows since $X_k = X_k^2$ and $Z_k = X_kZ_k$ for any $k$.
\end{proof}

\end{document}